%% file: main.tex
\documentclass[conference]{IEEEtran}

\usepackage{balance}
\usepackage{enumitem}
\usepackage{ulem}
\usepackage{algorithm}
\usepackage{algorithmic}
\usepackage{cite}
\usepackage{amsmath,amssymb,amsfonts}
\usepackage{graphicx}
\usepackage{xcolor}
\usepackage{amsthm}

\newtheorem{theorem}{Theorem}

\newtheorem{lemma}[theorem]{Lemma}

\begin{document}

\title{Comprehensive and Efficient Data Labeling via Adaptive Model Scheduling}


\author{
\IEEEauthorblockN{Mu Yuan, Lan Zhang, Xiang-Yang Li}
\IEEEauthorblockA{\textit{University of Science and Technology of China} \\
Hefei, China \\
ym0813@mail.ustc.edu.cn, zhanglan@ustc.edu.cn, xiangyangli@ustc.edu.cn}
\and
\IEEEauthorblockN{Hui Xiong}
\IEEEauthorblockA{\textit{Rutgers University} \\
Newward, USA \\
xionghui@gmail.com}
}

\maketitle

\begin{abstract}
Labeling data (e.g., labeling the people, objects, actions and scene in images) comprehensively and efficiently is a widely needed but challenging task. 
Numerous models were proposed to label various data and
many approaches were designed to enhance the ability of deep learning models or accelerate them. 
Unfortunately, a single machine-learning model is not powerful enough to extract various semantic information from data.
Given certain applications, such as image retrieval platforms and photo album management apps,
 it is often required to execute a collection of models to obtain sufficient labels.
With limited computing resources and stringent delay, 
  given a data stream and a collection of applicable resource-hungry deep-learning models,
 we design a novel approach to adaptively schedule a subset of these models to execute on each data item, 
 aiming to maximize the value of the model output (e.g., the number of high-confidence labels).
Achieving this lofty goal is nontrivial since a model's output on any data item is content-dependent and unknown until we execute it.
To tackle this, we propose an \textit{Adaptive Model Scheduling} framework, consisting of
1) a deep reinforcement learning-based approach to predict the value of unexecuted models by mining semantic relationship among diverse models, and
2) two heuristic algorithms to adaptively schedule the model execution order under a deadline or deadline-memory constraints respectively.
The proposed framework doesn't require any prior knowledge of the data, which works as a powerful complement to existing model optimization technologies.
We conduct extensive evaluations on five diverse image datasets and 30 popular image labeling models 
to demonstrate the effectiveness of our design: our design could save around 53\% execution time without loss of any valuable labels.
\end{abstract}





\section{Introduction}
\label{sec:intro}
\input{version2/intro.tex}

\section{Data-driven Analysis}
\label{sec:potential}
\input{version2/potential.tex}

\section{Problem \& Design Overview}
\label{sec:submod}
\input{version2/submodular.tex}

\section{Model Value Prediction}
\label{sec:semantic}
\input{version2/semantic.tex}

\section{Adaptive Model Scheduling}
\label{sec:adaptive}
\input{version2/scheduling.tex}

\section{Evaluation}
\label{sec:exp}
\input{version2/exp.tex}

\section{Related Work}
\label{sec:related}
\input{version2/related.tex}

\section{Conclusion \& Future Work}
\label{sec:conclusion}
\input{version2/conclusion.tex}


\balance

\bibliographystyle{abbrv}
\bibliography{main}

\end{document}

%% file: version2/intro.tex
With the explosive growth of data volume and the rapid development of the AI industry, it is an appealing task to comprehensively label large amounts of data as fast as possible. 
For example, annotating each image with a comprehensive collection of semantic labels can power a wide variety of functionalities, such as multi-label image retrieval and image classification \cite{li2017deep,zhang2018instance,zhao2015deep}, and provide more possible avenues of making the most of images. 
Image retrieval platforms usually employ a series of machine learning models to provide as many labels as possible to describe each image to improve the quality of search results. By identifying faces, landmarks, scenes, objects, emotions and events in photos, many smartphones support thousands of searchable keywords to find photos in one's album.
On data trading platforms~\cite{li2018can}, the richer the label of a data set, the higher the price of the data set.
To facilitate the extraction of rich labels, there are two main streams of previous work.
One stream of efforts have been devoted to enhancing a single model's capability.
Multi-label learning~\cite{han2015learning,yang2018complex} and multi-task learning~\cite{kaiser2017one,ma2018modeling} have been proposed to enable a single model to extract more complex semantics of input data.
The other stream of work focus on accelerating the model execution by designing a variety of methods, including model compression via parameter pruning \& sharing~\cite{chen2016compressing}, 
network architecture optimization~\cite{wu2017squeezedet} 
and adaptive model configuration~\cite{jiang2018chameleon}.

Despite the efforts made by existing work to provide us with increasingly powerful and lightweight deep learning models for various tasks,
there still exist several major limitations hindering comprehensive and efficient data labeling for many applications.
(1) The limited ability of a single model: 
 one single model can usually only capture features of certain aspects of the data.
 Thus, in many cases, e.g., image retrieval, image analysis for video surveillance, and data trading, it is inevitable to execute a series of deep learning models to achieve a broad understanding of data.
Usually, for these tasks, the more diverse models utilized, the better the services provided. 
(2) Limited computing resources: 
one simple policy is to execute all possible applicable deep learning models on each piece of input data. Unfortunately, this is infeasible in many cases due to the limited computing resources and the intolerable long delay. 
For example, an advanced human pose estimation model~\cite{cao2018openpose} consumes more than 10GB GPU memory. 
(3) Limited expertise of users: facing a large number of available deep learning models and massive raw data, it is quite difficult, if not impossible, for a user to select an optimal subset of models for achieving
 a comprehensive labeling as well as a low computing overhead. To tailor different deep learning models for different input data is also a very difficult and time-consuming task.
(4) Diverse and unknown content of raw data:
even if an oracle can select a series of models suitable for the upcoming data, executing all those models for each piece of input data can still result in serious computing waste.
As the example shown in Fig.~\ref{fig:intro}, though every model outputs some valuable labels for some input images, 16/30 model executions didn't generate anything useful.
It is nontrivial to avoid such waste since the content of the data is unknown before model executions. 
As a result, given a set of models, how to extract the maximum value (i.e., as many high-confidence labels as possible) from large-scale data at a minimum cost remains a very challenging problem.

\begin{figure}[t]
\begin{center}
	\includegraphics[width= 0.85\linewidth]{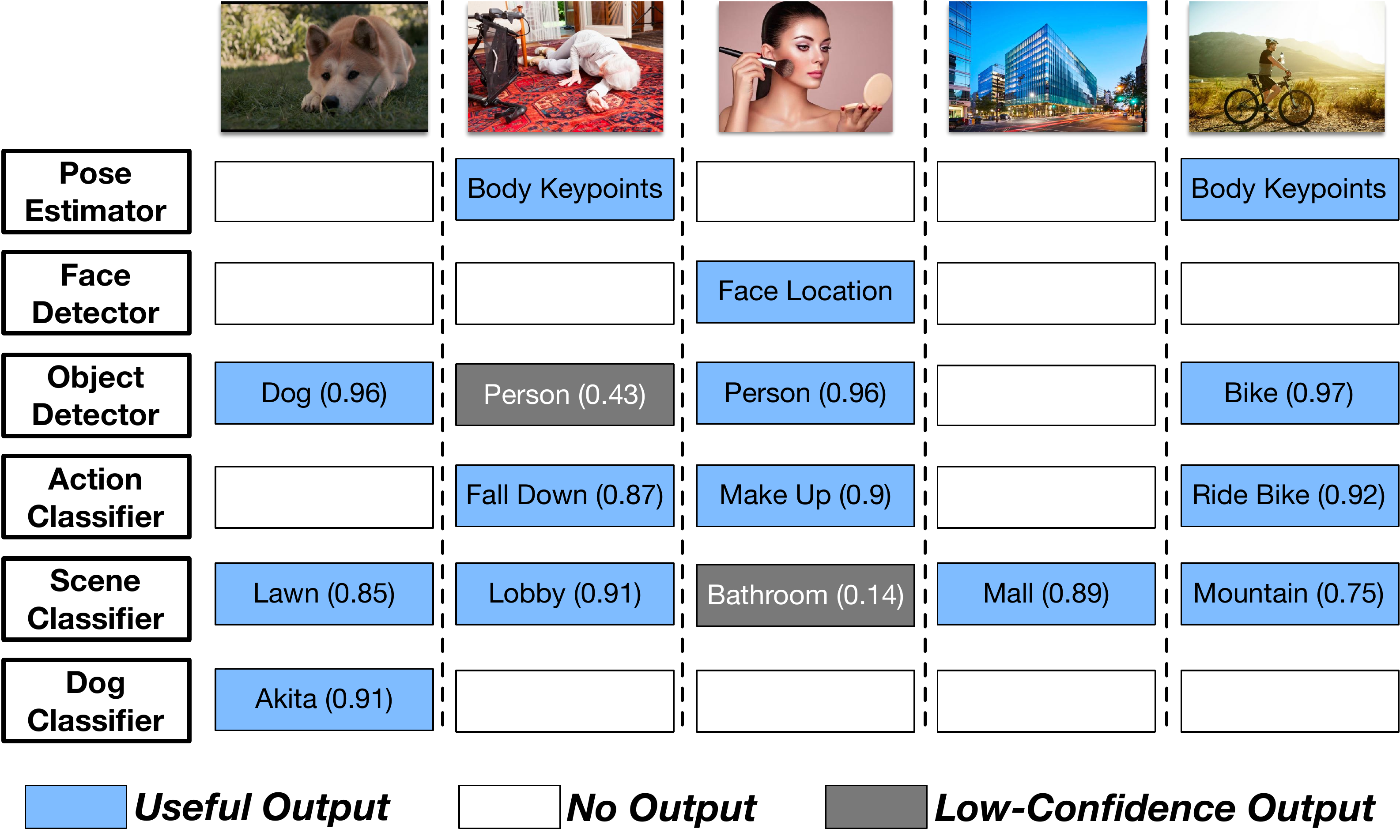}
	\caption{Output of six models executed on the same set of images from a public dataset. 
	}
	\label{fig:intro}
	\vspace{-0.35in}
	\end{center}
\end{figure}

The blue boxes in Fig.~\ref{fig:intro} show the significance of executing multiple diverse models to achieve a broad understanding of data, while the white and grey boxes reveal that a large portion of computing resources could be wasted without careful model selection. 
Towards comprehensive and efficient data labeling, in this work we propose a novel \textit{adaptive model scheduling} framework, which works orthogonally to model enhancing and accelerating methods.
Specifically, given a set of deep learning models and a stream of raw data to be processed, we propose to adaptively schedule model execution on each piece of input data to maximize the value of extracted labels under some constraints on computing resources and delay.
This ambitious goal, however, is nontrivial to achieve due to the following reasons.
\textbf{First}, an optimal scheduling policy should only select models that will output valuable labels, i.e., labels with high confidence, for each piece of data. 
A model's output, however, highly depends on the content of the data, which is generally unknown in advance and hard to predict.  
As the example in Fig.~\ref{fig:intro}, 
 the ideal policy should only conduct those blue-box model executions. 
Without knowing if there is a human body in an image, however, it is difficult to decide whether or not to execute the computation-intensive pose estimator.
\textbf{Second}, for each piece of raw data, even if we can somehow obtain the subset of models that will generate valuable labels, it is still theoretically NP-hard to schedule them to maximize the value of outputs under multi-dimensional constraints.
The most relevant well-defined problem is the submodular function maximization with multi-dimensional knapsack constraints,
 however, different from which the values of model outputs are unknown in our model scheduling problem.
Even when all output values are given in advance, which is the simplified version of our case, the problem remains NP-hard.  

Given a stream of data, a simple case is that the data could be partitioned into chunks (such as segments of video), and each chunk 
has a certain correlation in its content.
For such kind of data, our extensive evaluations confirmed that a simple exploration-exploitation solution works extremely well: at the beginning of the chunk, we explore almost all possible models and find the best subset
 of models for this chunk, then for the remaining of this chunk, we run this subset of models, i.e., exploitation.
The challenging part is the other case when the data items are not correlated, e.g., a set of randomly taken images.
To tackle such a question, 
 we explore the implicit semantic relationship among several deep learning models, which is inspired by the observation that some executed models' outputs provide hints of other models' outputs.
For example, if a pose estimator fails to detect any body keypoint, then a face detector is very likely to output nothing and should not be executed.
A straightforward solution is to manually design such execution rules to characterize the dependency among different output labels.
Our experimental results show that handcrafted rules are cumbersome and inefficient when there are tens of models, which support more than one thousand labels in our experiments.

\textbf{Contributions:}
In this work, we propose an adaptive model scheduling framework, consisting of
1) a novel deep reinforcement learning (DRL) based method to model the semantic relationship among models and predict the output of unexecuted models based on outputs of executed ones, and
2) two heuristic scheduling algorithms to maximize the value of output labels under a deadline or deadline-memory constraints.
Our adaptive model scheduling framework gathers the power of existing models to achieve comprehensive and efficient labeling of large-scale data.
It does not require users to have deep learning related expertise or know the content of the raw data in advance.
To the best of our knowledge, this work is the first step to explore the relationship among multiple models.
We conduct extensive evaluations on large-scale and highly diverse images from five public datasets using 30 popular deep learning models for 10 visual analysis tasks.
Our experimental results show that the proposed adaptive model scheduling framework could save 53.1\% execution time when we need a 100\% recall of all valuable labels.
We could save about 70.0\% execution time when we only need  80\% recall of all valuable labels.
Given the 0.5s delay budget for each image, our proposed algorithms could improve the obtained output value by 132-310\% compared with the randomly scheduling of models.


%% file: version2/potential.tex
This work is motivated by increasing demands for comprehensive and efficient data annotation.
The main bottlenecks of this task are the limited ability of a single model and limited computing resources.
Instead of enhancing the ability of every single model or accelerating its execution, we propose to select the optimal combination of off-the-shelf models to be executed for each piece of raw data, which is to schedule the execution of models adaptively.
To demonstrate the necessity and potential of this work, we conduct a data-driven analysis of comprehensive image labeling for applications like image retrieval and photo management.
We collected 394,170 images from three well-known public datasets (MSCOCO 2017~\cite{lin2014microsoft}, Places365~\cite{zhou2017places} and MirFlickr25~\cite{huiskes2008mir}).
MSCOCO dataset is collected for the object detection task; Place365 dataset is collected for the place classification task; MirFlickr25 consists of images from the social photography site Flickr.
We believe the content of these three datasets is diverse enough to represent most real-life cases.
We deployed 30 deep learning models for 10 different visual analyzing tasks, such as objection detection, face detection, action classification, and place classification.
In all, these models extract 1104 labels.
See Section~\ref{sec:exp}.A and Table~\ref{tab:alltask} for more details for datasets and models.


We executed all 30 models on every image and collected the output results,
consisting of labels and their corresponding confidences.
We refer to such an ``executing all models on each image'' procedure as ``no policy''.
We analyze the drawbacks of the naive ``no policy'' model execution. Using 4 Tesla P100 GPU, it took 5.16s on average to process one image on a GPU card and about \textbf{6 days} to analyze all 394,170 images.
The cost is unacceptable for many delay-sensitive tasks and will increase significantly as the number of images/models grows. 
Fortunately, not all computing costs are necessary.
Fig.~\ref{fig:intro} reveals that there is a significant waste of computing resources for the ``no policy'' model execution since a large portion of model executions only generate low confidence outputs or even nothing.
To further quantify the waste,
 we simulate a process of data labeling with an ideal ``optimal policy''.
Based on the complete execution results we obtained, the ``optimal policy'' only selects the model executions which generate high-confidence outputs.
As shown in Fig.~\ref{fig:dda}, the ``optimal policy'' requires 1.14s to process one image on average, which costs only 22.1\% time of the ``no policy''.  It speeds up the total labeling process from 6 days to about 1.3 days.

\begin{figure}[t]
	\begin{minipage}{.5\linewidth}
		\centering
		\includegraphics[width=\linewidth]{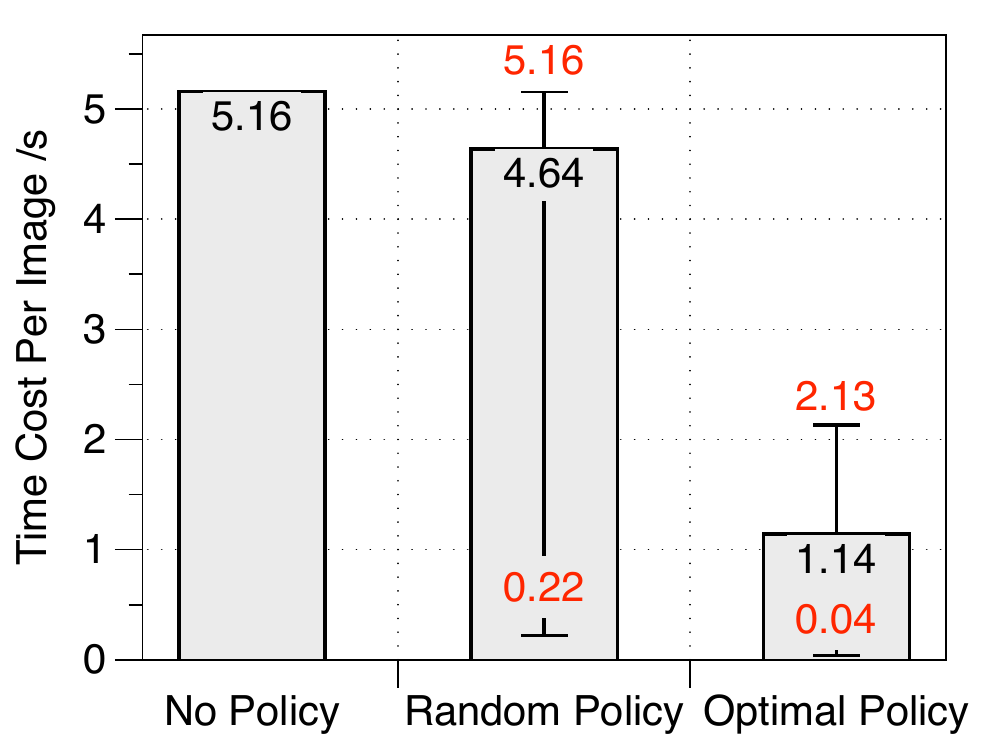}
	\end{minipage}\hfill
	\begin{minipage}{.5\linewidth}
		\centering
		\includegraphics[width=\linewidth]{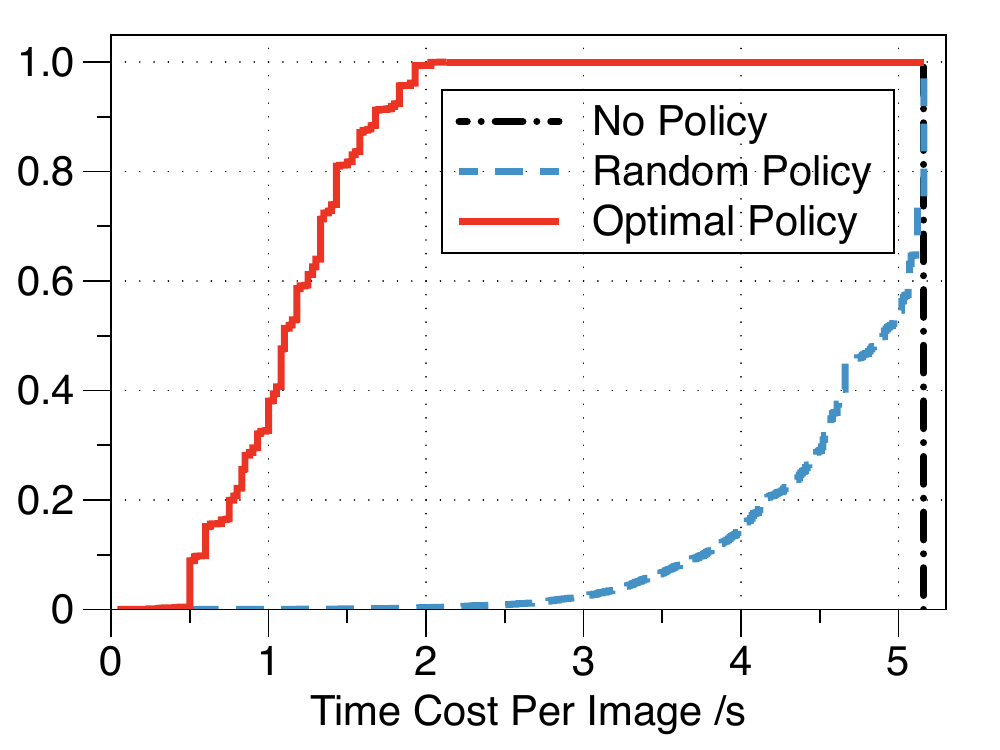}
	\end{minipage}
	\caption{Time cost of three policies to obtain all valuable labels for each image (Left: average time cost per image; Right: CDF of the time cost per image).}
	\vspace{-0.2in}
	\label{fig:dda}
\end{figure}

Our data-driven analysis shows that leveraging a set of diverse models can greatly enrich the data annotation without requiring any preknowledge of the raw data or expertise of deep learning models. An optimal model execution policy can save up to \textbf{77.9\%} computing cost without the loss of any valuable labels. The significant benefits and huge improvement potential strongly motivate us to search for the optimal policy.  
To design an optimal model scheduling policy, however, is nontrivial due to the reason that it is very challenging to estimate a model's output for a piece of raw data before the actual execution.  
Here, we simulated a ``random policy'' using the complete outputs of ``no policy'', which executes models in a random order for each data until all valuable labels for this data have been recalled.
Fig.~\ref{fig:dda} shows that the ``random policy'' does not help much, which only saves 0.52s per image on average. 
In the CDF of time cost per image, we can see a huge gap between the ``optimal policy'' and the ``random policy''.
It reveals that an effective model scheduling policy requires a much more insightful design.

%% file: version2/submodular.tex
In this work, we aim to design such an optimal policy that adaptively schedules the execution of a set of off-the-shelf  deep learning models for each piece of input data to maximize the value of outputs under computing resources and delay constraints.
For the policy to be widely applicable, it doesn't need any modification of the existing deep learning models, nor knowledge of the input data content in advance.
In this section, we first present the formal definition of the problem and then introduce the overview of our design.

\subsection{Problem Definition}
Given a set of available models $M$, let $L(m)$ denote the supported labels for $\forall m \in M$.
Then the whole set of labels is $L(M)=\bigcup_{m \in M} L(m)$.
Each label $l_i \in L(M)$ has a profit $p_i > 0$, which indicates the value of the label $l_i$ to the user.  
For an input data $d$, let $O(S, d)$ denote the output of a subset of models $S \subseteq M$ executed on the data $d$, which is a set of labels.
We evaluate the output by the following function:
\begin{equation}
\label{eq:eval}
f(S,d) = \sum_{l_i \in  O(S,d)} p_i
\end{equation}
The objective is to select such a model subset $S$ that maximizes the evaluation function $f(S, d)$, while the computing costs of $S$ are under some constraints, which can be formalized as follows:
\begin{equation}
\label{eq:obj}
\begin{aligned}
& \max_{S \subseteq M} f(S,d)\\
\text{s.t. } & \text{constraints on } S
\end{aligned}
\end{equation}
Some common constraints include time delay caused by executing $S$ and GPU memory occupied by $S$. 

\begin{lemma}
The evaluation function $f$ is a non-negative, non-decreasing submodular function.
\end{lemma}
\begin{proof}
The proof is omitted due to the triviality.
\end{proof}
    
    
    

Note that, if the supported labels of different models do not overlap, the evaluation function becomes a modular function, i.e. the inequality characterizing submodularity holds with equality.
Traditional submodular function optimization problems commonly assume that the cost of the evaluation function itself is negligible, e.g., the weighted coverage function.
In our problem, however, to calculate the evaluation function $f$ requires actual execution results of computation-intensive deep learning models.
But once we have executed those models, there is no need for scheduling them anymore.
This dilemma makes classic submodular function optimization approaches~\cite{krause2014submodular, sviridenko2004note, golovin2011adaptive} infeasible to solve our problem.
Therefore, the biggest challenge here is to estimate the evaluation function before the model execution.
Moreover, even we somehow know models' outputs in advance, how to schedule the execution of them to maximize the evaluation function under multi-dimensional constraints remains a hard problem.

\begin{figure}[t]
\centering
	\includegraphics[width=0.9\linewidth]{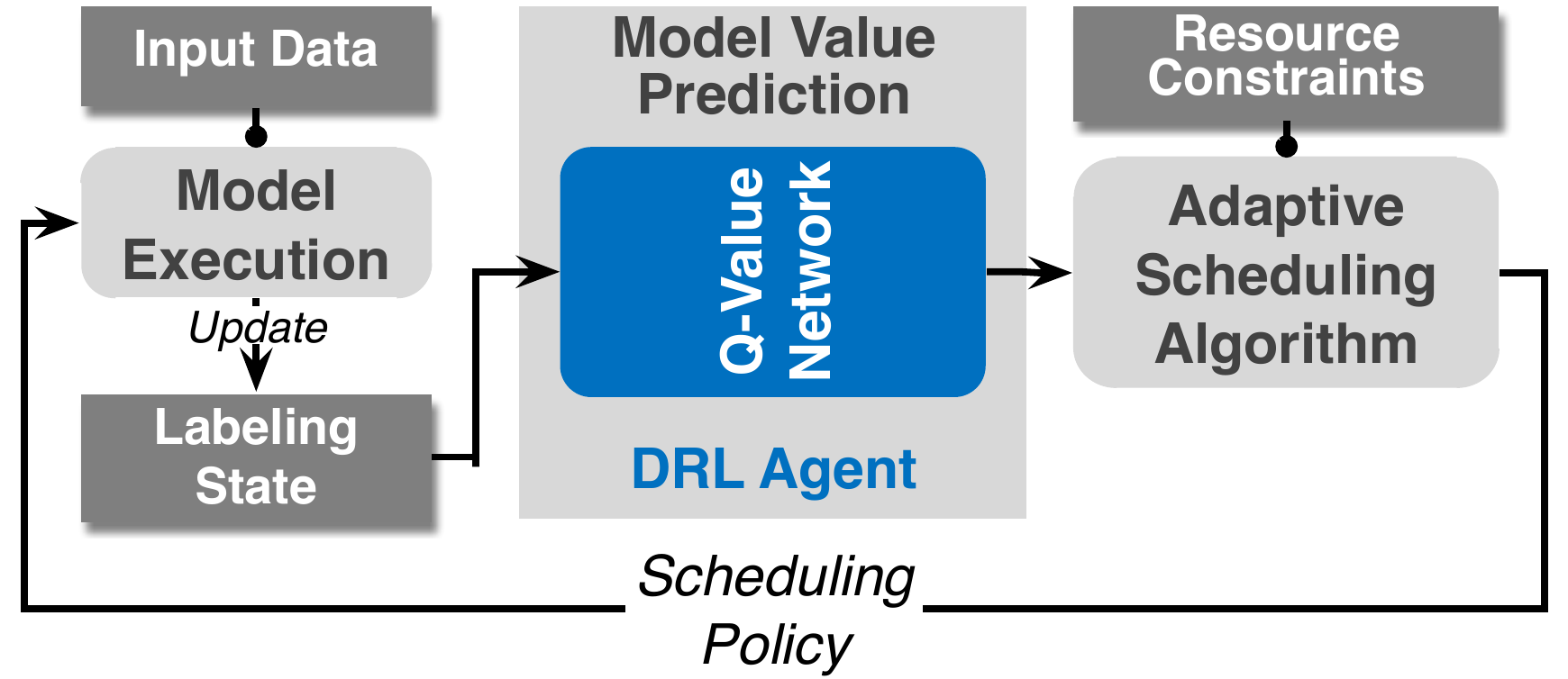}
	\caption{Overview of our proposed \textit{adaptive model scheduling} framework. }
	\label{fig:overview}
	\vspace{-0.15in}
\end{figure}

\subsection{Core Idea and Design Overview}
\label{subsec:overview}

Based on the observation that executed models' outputs could provide hints for the outputs of unexecuted ones, we propose to utilize the implicit \textit{semantic relationship} among labeling capacities of models to estimate the evaluation function before the actual execution.

\subsubsection{Handcrafted rule-based solution.}
A straightforward solution to characterize the semantic relationship is to manually design some model execution rules based on common sense, which increase (or decrease) the probabilities to execute models that ought to (or not to) output valuable labels.
For instance, if the label ``undersea'' is returned by a scene classifier with high confidence, the probability of running a fish classifier should be increased, while a household object detector should be assigned with a low probability for execution.
Table~\ref{tab:rules} gives more examples of handcrafted rules. 
The number of models and the supported labels, however, could be too large to design a set of effective rules.
In our implementation, 30 deep learning models support 1104 labels in total. 
Designing execution rules for a system at that scale is cumbersome and ineffective. 
The evaluation results in Section~\ref{subsec:exp3} show that handcrafted rules slightly improve the performance compared with the ``random policy'', but still leaves a large room for optimization.

\subsubsection{Reinforcement learning-based framework.}
We propose a novel deep reinforcement learning (DRL) based approach to better characterize the complex semantic relationship among models and predict the value of unexecuted models.
With the estimated value of the evaluation function, we design scheduling algorithms to optimize the model execution under different constraints. 
As shown in Fig.~\ref{fig:overview}, our proposed framework consists of two components:

\textbf{(1) Model value prediction.}
A DRL model is trained to map the labeling state, i.e., outputs of currently executed models, to values of unexecuted models.
When each piece of input data arrives, the labeling state is initialized to an empty set. 
In each iteration, the DRL agent predicts the values of remaining unexecuted models based on the current state.   

\textbf{(2) Adaptive model scheduling.}
According to the predicted values of unexecuted models, the adaptive scheduling algorithm determines the model (or a set of models) to execute in this iteration based on the current computing resources and/or delay constraints.
Two common types of constraints, delay and GPU memory, are considered in this work.
The labeling state is updated according to the actual model execution results and a new iteration begins.

This ``prediction-scheduling-execution'' loop continues until a stopping condition is reached, e.g., the running time reaches the deadline constraint.
Our two-components framework disentangles models' semantic relationship from their computing resource costs.
The model value prediction component only cares about whether the model will output valuable labels or not, leaving the costs issue to the scheduling algorithms.
In this way, the learned semantic relationship can be adopted flexibly in various situations with different resource constraints.
In the following sections, we present the design of each component in detail. 

%% file: version2/semantic.tex
The first key component of our proposed adaptive model scheduling framework is to predict the output values of unexecuted models based on the outputs of executed ones.
The prediction accuracy is critical to the subsequent scheduling algorithms.
To confirm the interactive characteristic of the proposed framework, we propose a DRL based method for the mode value prediction task.
Deep reinforcement learning has been applied in many fields, including robotics control~\cite{kober2013reinforcement}, video games~\cite{mnih2013playing}, system tuning\cite{zhang2019end}, etc.
To our best knowledge, this work is the first step to explore the DRL for mining semantic relationships among labeling capacities of multiple models. 

There are three main parts of a DRL method, including the environment observation, the action space, and the reward feedback.
In our problem, we model the \textit{labeling state} as the environment observation, which is a $n$-dimension binary vector ($n$ is the number of supported labels, i.e., $|L(M)|$).
The $i$-th binary value 1 or 0 indicates the state that the $i$-th label has been or has not been output by executed models.
Each model in $M$ is regarded as an action and the complete set of models compose a $|M|$-dimension discrete action space.
In every iteration, after the DRL agent selects an action/model $m\in M$, the system executes the model on the input data $d$.
According to the execution result $O(m,d)$, the labeling state is updated and the agent receives a reward.
The reward function is the key to the performance of the DRL agent.
We cannot directly adopt Eq.~(\ref{eq:eval}) as the reward function since it evaluates the final value of the complete output set when the execution stops.
The DRL agent, however, needs a reward function for intermediate feedback to each selected action.

\subsection{Reward Function}
The confidence $l_i.conf$ represents the model's judgment of the accuracy of label $l_i$, 
which can be adopted as the profit $p_i$.
However, based on the confidence of each output label, directly using Eq.~(\ref{eq:eval}) as the reward function has some shortcomings.
The first problem is that the number of output labels influences the reward overly, which varies among different deep learning models.
As an example, an action classification model usually outputs only one label, while a face landmark detector outputs at most 70 keypoint labels for each face in the image.
In this case, the agent would easily tend to prefer the face landmark detector regardless of the data content, due to the overwhelming returned value.
We use a logarithmic function to mitigate the bias caused by different numbers of models' output labels.
Our experimental results show that other smoothing functions such as the average confidence of a model's output labels can achieve a similar effect as long as the values of different models are in the same order of magnitude.
On the other hand, considering the different requirements of diverse applications for model priorities, we introduce a parameter $\theta_m$ as the user-defined priority for a model $m$.
The greater $\theta_m$ is, the higher priority the model $m$ has.
For example, in a surveillance video monitoring system, when computing resources are insufficient, an abnormal event detection model should have a higher priority than a common object detector.
Experimental results in Section~\ref{subsec:exp5} show that by adjusting the parameter $theta$, we can effectively control the model execution priority without sacrificing the optimization performance.

Taking all the aforementioned factors into consideration, 
 the reward function is defined as:
\begin{equation}
{\small
\label{eq:reward}
r(m,d) =\left\{ 
\begin{aligned}
&\ln(\theta_m \sum_{l_i\in O'(\{m\},d)}{l_i.conf} + 1) &,O'(m,d)\neq\emptyset\\
&-1 &,O'(\{m\},d) = \emptyset
\end{aligned}
\right.
}
\end{equation}
We define $O'(m,d)$ as the set of new labels output by a running model $m$, which have not been output by other models yet.
Since the supported labels of different models may overlap, $O'(\{m\},d) \subseteq O(\{m\},d)$.
To avoid selecting duplicated or valueless models, when $O'(\{m\},d) = \emptyset$, the agent receives a punishment $-1$ as the feedback value.

\subsection{Agent Learning}
\label{subsec:agent}
The DRL agent needs to learn the mapping from the environment observation to the action space, which is the mapping from the current labeling state to the selection of the model to be executed.
The traditional table-based approach cannot handle this task since the mapping space increases exponentially with the number of labels and models.
Facing the complexity of the problem, we adopt a deep Q-value network (DQN)~\cite{mnih2013playing} as the observation-action mapping function.
The DQN architecture can be adjusted to adapt to different sizes of observation and action spaces.
In our implementation, a hidden dense layer with 256 neurons activated by ReLu is used to cope with a 1104-dimension observation space and a 30-dimension action space.
For training the agent, we implement four popular DRL approaches: Original DQN~\cite{mnih2013playing}, Double DQN~\cite{van2016deep}, Dueling DQN~\cite{wang2015dueling} and Deep SARSA~\cite{corazza2015q}.
We compare their performance in Section~\ref{subsec:exp2}.
Theoretically, the proposed DRL-based agent can be trained by any Q-value network-based DRL approach.

\textbf{END action.}
The classic epsilon-greedy policy is applied in the training process, which either selects the action with the maximal Q value or randomly selects an action with the probability epsilon.
In this way, however, the training process is hard to reach convergence.
The reason is that, for each input, after some iterations, the agent will reach the state that all valuable labels have already been recalled.
So according to the reward function Eq.~\ref{eq:reward}, any further action will bring back a punishment, which obstructs the agent against improving subsequent actions.
To fix this problem, we add an \textbf{``END''} action, whose reward is zero, to indicate the end of the selection process for the current input.
Then the agent has the option to select the ``END'' action to avoid further punishment when no valuable model is left.
Results of experiments on agent learning confirm that the ``END'' action effectively quickens the velocity of convergence.
Note that, the ``END'' action is only used in the training process.
In the scheduling process, the agent stops when certain computing resource runs out, e.g., running time exceeds the delay budget.

%% file: version2/scheduling.tex
The second key component of our framework is adaptive scheduling algorithms under various computing resources constraints.
Given a set of available deep learning models and input data, the trained DRL agent provides the value prediction for unexecuted models based on the current labeling state.
Based on the predicted values and remaining resources, a scheduling algorithm determines the execution policy that aims to maximize the evaluation function $f(S,d)$.
It is worth mentioning that if there is no resource constraint, the DRL agent itself can schedule the model execution in a greedy manner which selects the remaining model with maximal Q-value at each iteration and terminates when all models have been executed.
In this work, we consider the two most common constraints for data labeling tasks: deadline and limited memory.
We first analyze the scheduling problem in a single-processor setting with a deadline constraint and then in a multi-processor setting with both deadline and memory constraints.

\subsection{Deadline Constraint}
In a single processor case, models can only be executed serially.
The deadline constraint is set for each input data, which is a common requirement of delay-sensitive systems.
For example, a video surveillance system needs to respond to suspicious events or objects in seconds.
With the deadline constraint, our optimization problem defined in Eq.~(\ref{eq:obj}) is specified as follows:
\begin{equation}
\begin{aligned}
& \max_{S \subseteq M} f(S,d)\\
\text{s.t. } &\sum_{m \in S} m.time \leq B_{time}
\end{aligned}
\end{equation}
, where $m.time$ is the execution time of model $m$ and $B_{time}$ is the constraint of the total execution time for all selected models in $S$.

The most relevant problem is the submodular function maximization with a knapsack constraint, which is NP-hard.
A commonly used heuristic approach to solve the knapsack problem is the cost-profit greedy algorithm, which selects the model $m$ that maximizes $\frac{f(S\cup \{m\},d) - f(S,d)}{m.time}$ at each iteration.
It is obvious that in the worst case, the performance of this algorithm can be arbitrarily bad.
A greedy algorithm combined with the partial enumeration of all subsets of size 3 has been proved to achieve a $1-1/e$ approximation~\cite{sviridenko2004note}.
In the traditional knapsack problem, the value of each item is known (easily calculated) and fixed.
In our problem, however, the real model value is unknown and we can only obtain the exact value after model execution.
Therefore, the enumeration-based algorithm is infeasible for our problem, not to mention its high computational complexity ($O(|M|^5)$)~\cite{sviridenko2004note}.
Moreover, the prediction of a model's value is not fixed, which changes with the labeling state.
The uncertainty and dynamics of the model value make the scheduling task particularly challenging. 
In this work, we propose to adopt the cost-profit greedy algorithm by using the Q value of each unexecuted model as its estimated profit at each iteration. 
Algorithm~\ref{alg:ddl} demonstrates the scheduling policy in detail. 
This efficient algorithm achieves near-optimal performance in our extensive evaluations ( see Section~\ref{subsec:exp6}), which also verifies the effectiveness of our DRL-based model value prediction method.

\renewcommand{\algorithmicrequire}{\textbf{Input:}}
\renewcommand{\algorithmicensure}{\textbf{Output:}}
\begin{algorithm}[t]
	\caption{Model scheduling under deadline constraint.}
	\label{alg:ddl}
	\begin{algorithmic}[1]
		\REQUIRE model set $M$, time budget $B_{time}$, DRL agent $Q$
		\ENSURE model subset $S$
		\STATE $S \leftarrow \emptyset$
		\WHILE {$B_{time} > 0$}
            \STATE{ Filter out models that $m.time > B_{time}$}
    		\STATE	$m^*  \leftarrow \arg\max\limits_{m\in M\backslash S} \frac{Q(m,d)}{m.time}$
    		\STATE	$S \leftarrow S \cup \{m^*\}$
    		\STATE	$B_{time} \leftarrow B_{time} - m^*.time$
		\ENDWHILE
		\RETURN $S$
	\end{algorithmic}
\end{algorithm}

\subsection{Deadline-Memory Constraint}
In a multi-processor case, multiple deep learning models can be executed in parallel on a shared-memory computing platform.
A two-dimension constraint, a deadline and memory constraint, is considered for each input data.
The tangible formulation of this optimization problem is:
\begin{equation}
\begin{aligned}
& \max_{S \subseteq M} f(S,d) \text{, where } S = \bigcup_{i=1}^N S_i\\
\text{s.t. } &\sum_{i=1}^N S_i.time \leq B_{time}\\
& \sum_{m \in S_i} m.mem \leq B_{mem} \ (\forall i, 1\leq i \leq N)
\end{aligned}
\end{equation}
Given an input, let the model scheduling process have $N$ iterations in total.
$S_i$ is the set of models being executed at the $i$-th iteration.  
Let $S_i.time$ denote the running time of the $i$-th iteration and $B_{time}$ denote the acceptable total execution time for all selected models $S$.
The memory cost of a model $m$ is denoted as $m.mem$, which is measured by the peak memory usage in our experiments.
At any time, the total memory usage of running models should not exceed the memory budget $B_{mem}$. 


\begin{algorithm}[t]
	\caption{Model scheduling under deadline-memory constraints.}
	\label{alg:ddl-mem}
	\begin{algorithmic}[1]
		\REQUIRE model set $M$, time budget $B_{time}$, memory budget $B_{mem}$, DRL agent $Q$
		\ENSURE model scheduling policy $S$
		\STATE $S \leftarrow [\ ]$, $TimeCost \leftarrow 0$, $S_t \leftarrow \emptyset$
		\WHILE {$TimeCost < B_{time}$}
    		\STATE Filter out models that $m.mem > B_{mem}$
    		\STATE	$m^*_1 \leftarrow \arg\max\limits_{m\in M \backslash S} \frac{Q(m, d)}{m.time \times m.mem}$
    		\STATE	$S_t \leftarrow S_t \cup \{m^*_1\}$
    		\STATE $B_{time}^{t} \leftarrow TimeCost + m^*_1.time$
    		\STATE Filter out models by temporary deadline $B_{time}^{t}$
    		\WHILE {$B_{mem} > 0$}
        		\STATE $m^*_2 \leftarrow \arg\max\limits_{m\in M \backslash S} \frac{Q(m,d)}{m.mem}$
        		\STATE $S_t \leftarrow S_t \cup \{m^{*}_2\}$
        		\STATE $B_{mem} \leftarrow B_{mem} - m^{*}_2.mem$
    		\ENDWHILE
    		\STATE $S.append(S_t)$
    		\STATE Wait until one model $m^{*}_3 \in S_t$ finishes execution
    		\STATE Update $TimeCost$
    		\STATE $B_{mem} \leftarrow B_{mem} + m^{*}_3.mem$
    		\STATE $S_t \leftarrow S_t \backslash \{m^{*}_3\} $
		\ENDWHILE
		\RETURN $S$
	\end{algorithmic}
\end{algorithm}

The optimization problem in the multi-processor setting is more challenging than that in the single-processor case.
The most related problem is the two-dimension orthogonal knapsack problem: given a set of squares, each of which has a profit, and the objective is to pack a subset of squares into a fixed-size square to maximize the total profit.
It is a NP-hard problem and several polynomial time approximation schemes (PTAS) have been proposed for this problem~\cite{han2008two}.
Due to the uncertainty and dynamics of model value prediction, however, those algorithms with performance guarantee are not feasible to solve our problem.
In this work, we design an efficient heuristic algorithm.
In each iteration, the algorithm first greedily selects one model that provides the highest value per unit resource area (the area here is the product of normalized time cost and memory cost) and sets the end time of this model as a temporary deadline.
Subject to the temporary deadline, the algorithm repeatedly selects the model with the highest value-memory ratio until the memory budget is reached.
See Algorithm~\ref{alg:ddl-mem} for details.
The intermediate model value is predicted by the pre-trained DRL agent.
Once a model completes its execution, its occupied memory will be released and a new iteration will begin.
We conduct extensive evaluations in Section~\ref{subsec:exp7} to show the effectiveness of our algorithm in practice.

\subsection{Performance Analysis}
\label{subsec:algo}
Due to the unknown real model value and the ever-changing predicted model value, existing algorithms for both the submodular function optimization with knapsack constraint and the two-dimension orthogonal knapsack constraints cannot be applied to solve our problem.
To measure the performance of our heuristic algorithms, we need to compare them with the optimal solution.
Under certain deadline and deadline-memory constraints, the NP-hard problem requires to enumerate $O(|M|!)$ policies to find the optimal one.
In our implementation, it is infeasible to enumerate all practicable execution policies of 30 models.
Therefore, we relax the problem to that even though the remaining resources are not enough for one model to complete its execution, the model can still be selected into the set $S$ and contribute the corresponding proportion of its value to $f(S,d)$.
We refer to the optimal policy of this relaxed problem as the \textit{optimal*} policy, which greedily selects the model with maximal $\frac{f(S \cup \{m\}, d) - f(S,d)}{m.time}$ within the deadline constraint or the model with maximal $\frac{f(S \cup \{m\}, d) - f(S,d)}{m.time * m.mem}$ within the deadline-memory constraint. 
The \textit{optimal*} policy provides a performance upper bound for the exact optimal policy of our original problem.
Experiment results in Section~\ref{sec:exp} show that the performance ratio between our heuristic algorithms and the ``optimal* policy'' exceeds 0.7 in most cases, which means that the true performance ratio of our algorithm is better than 0.7 in most cases. 
More evaluations of our algorithms will be presented in Section~\ref{sec:exp}.

%% file: version2/exp.tex
We implemented the proposed adaptive model scheduling framework and conducted extensive evaluations for large-scale  comprehensive image labeling tasks. 
Our framework can also be adapted for other types of data labeling tasks.
We will open source our code for use or modification.

\subsection{Experimental Setup}
\label{subsec:exp1}

\textbf{Deep learning models.} 
We consider 10 diverse visual analysis tasks and deployed a total of 30 popular deep learning models for these 10 tasks by directly utilizing the off-the-shelf pretrained weights or training models on public datasets.
These models can label images with a wide range of semantic information (1104 different labels in all).
Table~\ref{tab:alltask} summaries the deployed models and their supported labels.
For each model, the time cost ($m.time$) is set as the average value while the GPU memory cost ($m.mem$) is set as the peak value.

\begin{table}[ht]
	{\small{
		\begin{center}
			\begin{tabular}{p{5.5cm}|p{2cm}}
				\hline
				\textbf{Task Name} &\textbf{Label\#}\\
				\hline
				Object Detection~\cite{yolov3} & 80  \\
				Place Classification~\cite{zhou2017places} & 365 \\
				Face Detection~\cite{amos2016openface} & 1\\
				Face Landmark Localization~\cite{cao2018openpose}& 70\\
				Pose Estimation~\cite{xiu2018poseflow}& 17\\
				Emotion Classification~\cite{pylearn2_arxiv_2013} & 7\\
				Gender Classification~\cite{simonyan2014very} &  2\\
				Action Classification~\cite{carreira2017quo} & 400\\
				Hand Landmark Localization~\cite{simon2017hand}& 42\\
				Dog Classification~\cite{KhoslaYaoJayadevaprakashFeiFei_FGVC2011}& 120\\
				\hline
				\textbf{10 Tasks} & \textbf{1104 Labels}\\
				\hline
			\end{tabular}
			\vspace{0in}
			\caption{Summary of 10 visual analysis tasks.}
			\label{tab:alltask}
		\end{center}
        }
    }
    \vspace{-0.2in}
\end{table}

\textbf{Datasets and ground truth.} 
We conducted experiments on five public image datasets: 
1) Stanford40~\cite{yao2011human}, 
2) PASCAL VOC 2012~\cite{everingham2010pascal}, 
3) MSCOCO 2017~\cite{lin2014microsoft}, 
4) MirFlickr25~\cite{huiskes2008mir} and 
5) Places365~\cite{zhou2017places}.
To train the DRL agents and measure the effectiveness of our model scheduling system, it is necessary to obtain the ground truth of each model's performance.
We executed all 30 models on 5 datasets and stored the output labels and confidences.
For each dataset, we split it into a training set and a testing set with the ratio of 1:4.

\textbf{Hardware setting.}
For all evaluations, we employed a server with 48 Intel Xeon CPU E5-2650 v4 cores and one Tesla P100 GPU.

\begin{figure*}[t!]
\centering
	\begin{minipage}{.25\linewidth}
		\centering
		\includegraphics[width=\linewidth]{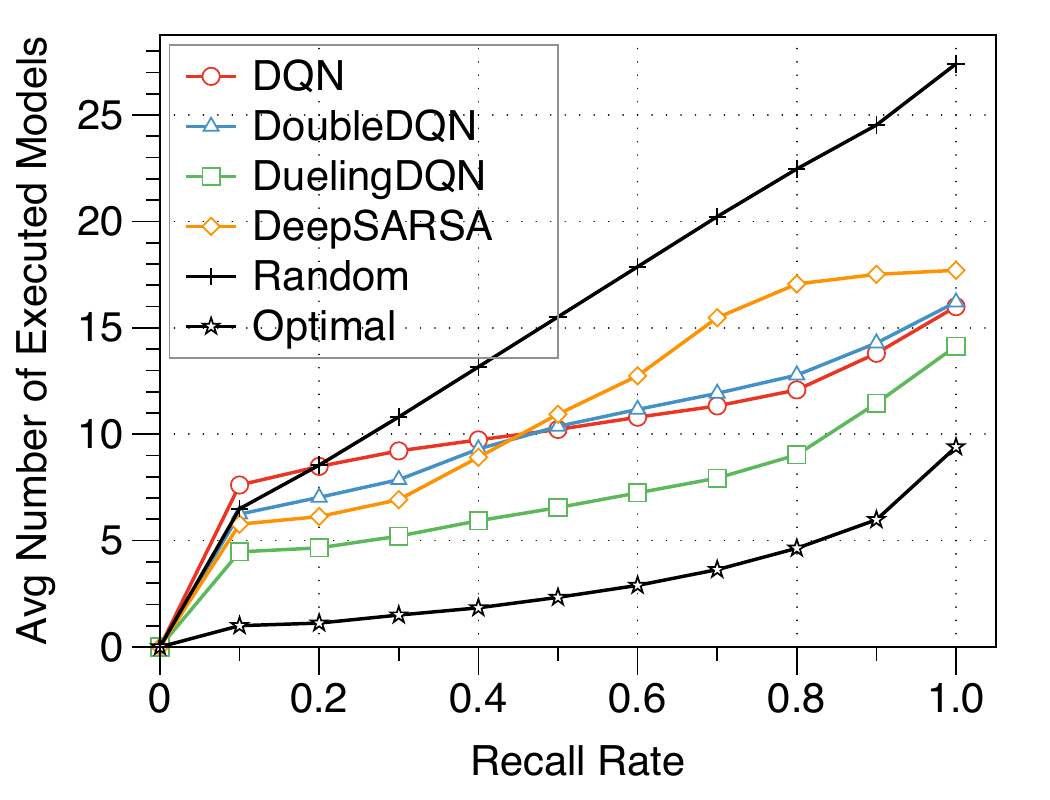}
		\textbf{\small (a) MSCOCO 2017}
	\end{minipage}
	\begin{minipage}{.25\linewidth}
		\centering
		\includegraphics[width=\linewidth]{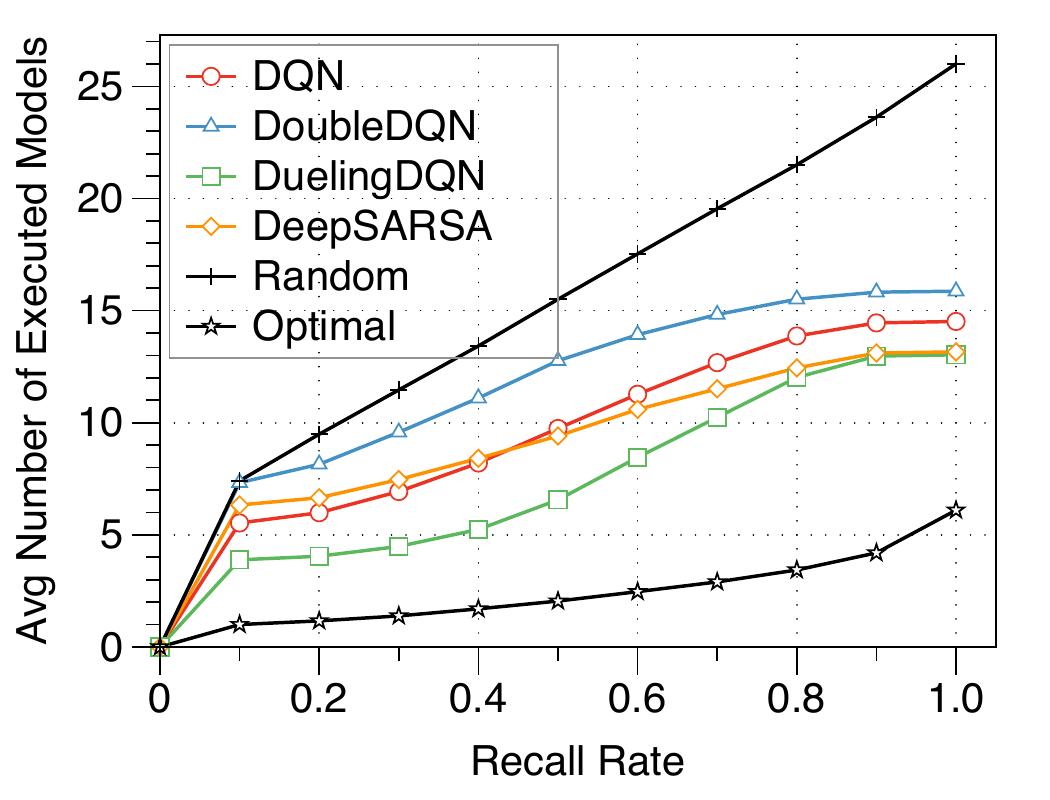}
		\textbf{\small (b) MirFlickr25}
	\end{minipage}
	\begin{minipage}{.25\linewidth}
		\centering
		\includegraphics[width=\linewidth]{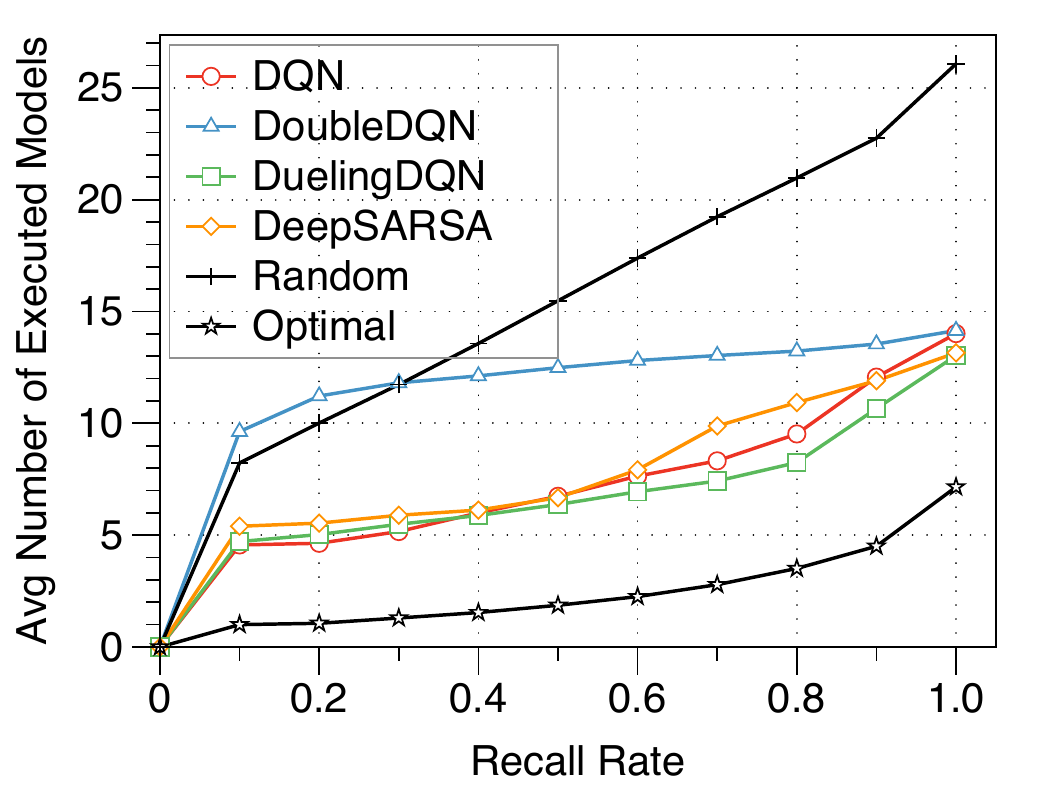}
		\textbf{\small (c) Places365}
	\end{minipage}
	\caption{The average number of executed models per image vs. required recall rate of output value.}
	\label{fig:rl1}
\end{figure*}

\begin{figure*}[t!]
\centering
	\begin{minipage}{.25\linewidth}
		\centering
		\includegraphics[width=\linewidth]{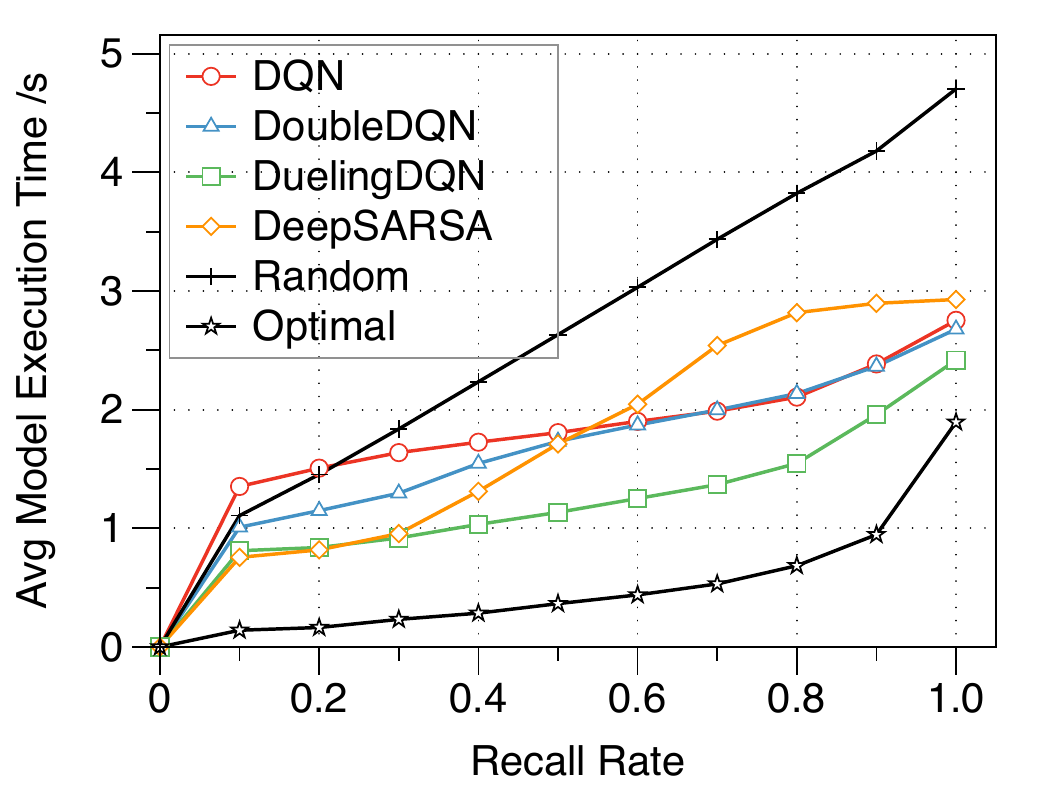}
		\textbf{\small (a) MSCOCO 2017}
	\end{minipage}
	\begin{minipage}{.25\linewidth}
		\centering
		\includegraphics[width=\linewidth]{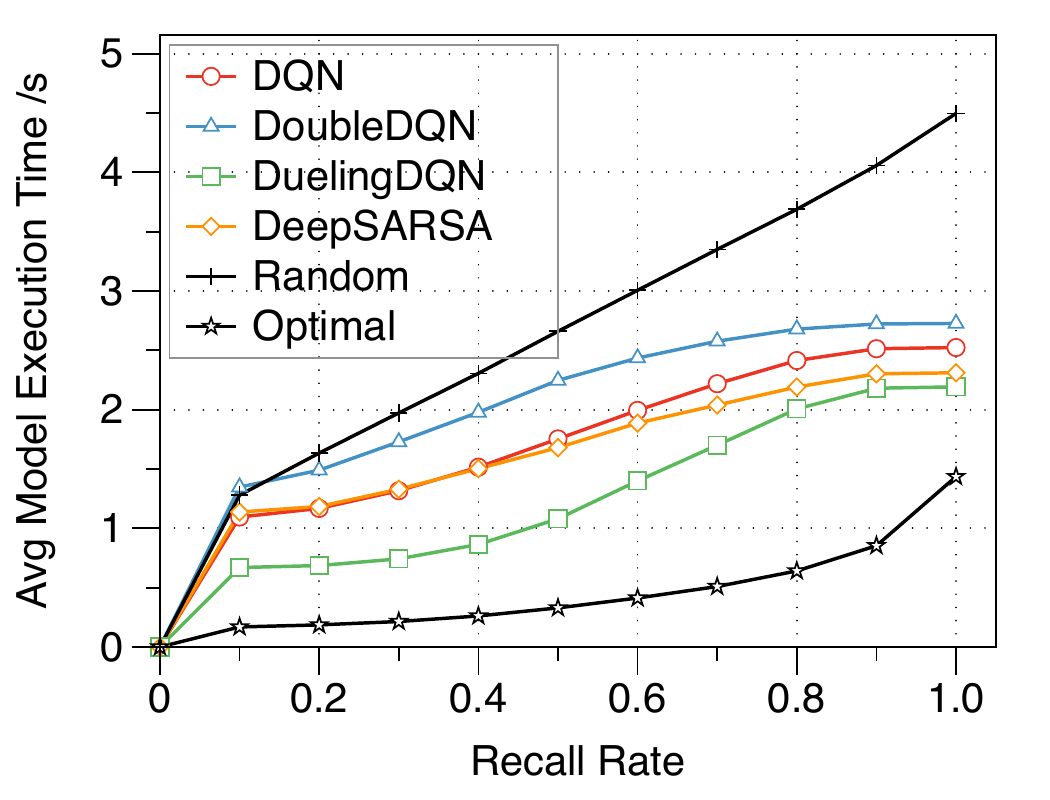}
		\textbf{\small (b) MirFlickr25}
	\end{minipage}
	\begin{minipage}{.25\linewidth}
		\centering
		\includegraphics[width=\linewidth]{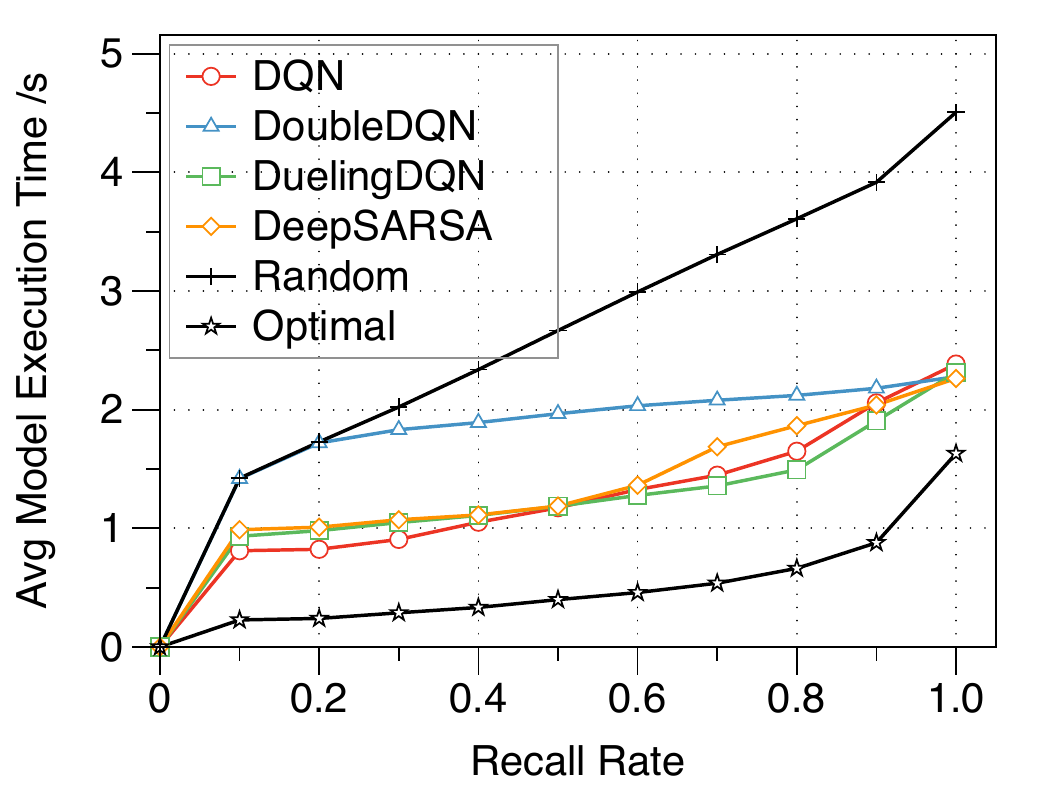}
		\textbf{\small (c) Places365}
	\end{minipage}
	\caption{The average time cost of executed models per image vs. required recall rate of output value.}
 	\vspace{-0.1in}
	\label{fig:rl2}
\end{figure*}

\subsection{RL-Based Model Value Prediction}
\label{subsec:exp2}
We implemented our designed DRL agent to predict model value before the execution.
We trained the DRL agent using four methods with the identical Q-value network as introduced in Section~\ref{subsec:agent}, including DQN~\cite{mnih2013playing}, DoubleDQN~\cite{van2016deep}, DuelingDQN~\cite{wang2015dueling} and DeepSARSA~\cite{corazza2015q}.
We trained and tested these four DRL agents on three datasets, MSCOCO 2017, MirFlickr25 and Places365, separately.
We cannot directly compare the model value predicted by DRL agents with the ground truth. 
The reason is that the predicted Q value of a model changes with the labeling state in each iteration, which is dependent on the set of currently executed models. 
For example, when the labels of a model have been output by other executed models, the predicted Q value of the model is -1 rather than the value of its output.        
To evaluate the performance of these agents, we conducted the following experiments.
We use the Q-value greedy policy that greedily selects the model with maximal Q value as the next one to be executed until the recall rate of true output value reaches a given threshold \footnote{The stop condition is determined by the ground truth}. 
Without considering any delay or resource constraint, the model selection is only dependent on the predicted model value.
The more accurate the predicted model value, the fewer the average number of executed models per image and the shorter the total execution time. 
Through these experiments, we compare the performance of DRL agents trained by four schemas.
To quantify the effectiveness of our DRL-based model value prediction method, we also implemented the following random policy and optimal policy as a comparison.
\begin{itemize}
    \item \textbf{Random policy}: randomly selects the next model, until the recall rate of true output value exceeds the threshold;
	\item \textbf{Optimal policy}: selects models in the descending order of their true output value, until the recall rate of true output value exceeds the threshold.
\end{itemize}

Fig.~\ref{fig:rl1} and Fig.~\ref{fig:rl2} show the experimental results of different policies on three diverse datasets, using the average executed number of models and average execution time as the metric respectively. 
The optimal policy consumes the smallest number of model executions to reach a given recall rate of output value since it knows the true value of each model.
There is a significant gap between the random policy and the optimal policy.
Compared with the random policy, the optimal policy saves 79.3-84.0\% (or 65.6-76.5\%) the number of model executions when the recall rate is 0.8 (or 1.0).
By predicting the model value, all four DRL agents outperform the random policy and effectively improve the model selection.
Among the four DRL agents, the one trained by DuelingDQN performs best, which saves 44.1-60.6\% (or 48.4-50.0\%) the number of model executions when the required recall rate is 0.8 (or 1.0), compared with the random policy.
Moreover, Fig.~\ref{fig:rl2} shows that, compared with the random policy,  the DRL agent saves 45.6-59.5\% (or 48.6-51.2\%) execution time with 0.8 (or 1.0) recall rate.
Besides, we notice that in both Fig.~\ref{fig:rl1} and Fig.~\ref{fig:rl2} the growth trends of model execution cost for DRL-based policies are similar to that of the random policy when the recall rate is under 10\%. 
As the recall rate gets larger, the growth trends for the DRL-based policies are more and more similar to that of the optimal solution. 
It meets the expectation of our design: the DRL agent initially selects models in an almost random way due to the lack of information about the data; 
after obtaining some execution results, the agent can predict the model value accurately and select the next model with the largest value as the optimal policy does.  


\vspace{-0.1in}
\subsection{Agent Knowledge vs. Handcrafted Rules}
\label{subsec:exp3}

\begin{figure}[t]
    \centering
    \begin{minipage}{.5\linewidth}
	    \includegraphics[width=\linewidth]{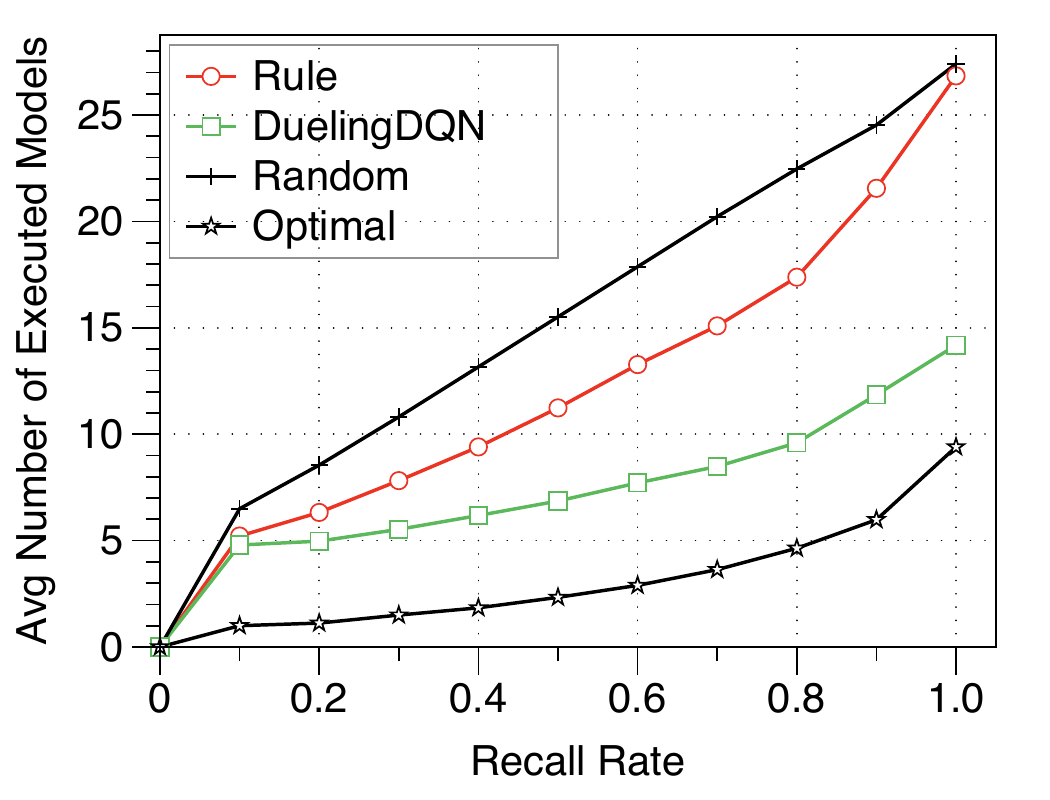}
    \end{minipage}\hfill
    \begin{minipage}{.5\linewidth}
	    \includegraphics[width=\linewidth]{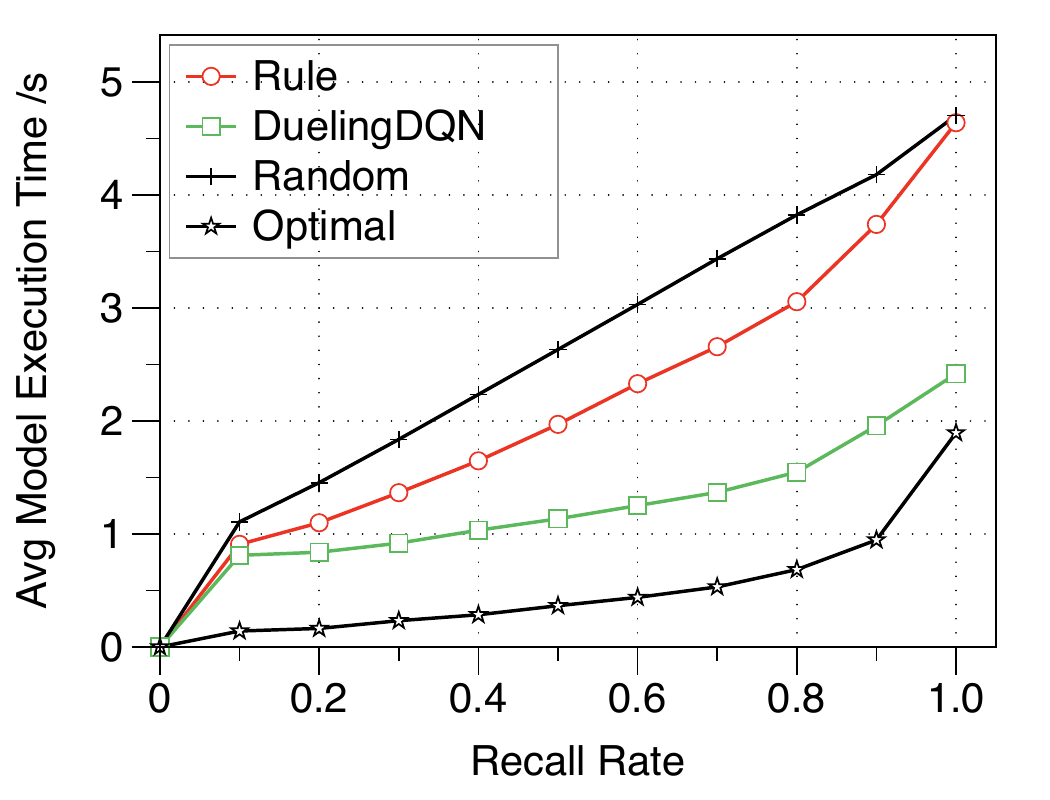}
    \end{minipage}
    \vspace{-0.1in}
	\caption{(Left) Average number of executed models and (Right) average model execution time, vs. required recall rate of output value.}
	\vspace{-0.1in}
	\label{fig:rule}
\end{figure}

A direct way to optimize model scheduling is to manually design execution rules based on common sense (see Section \ref{subsec:overview}).
To explore the effectiveness of handcrafted rules, we asked three volunteers with certain deep learning related knowledge to provide execution rules they can think of.
We selected 10 rules (listed in Table~\ref{tab:rules}) accepted by all of them and implemented a rule-based model scheduling policy.

\begin{table}[h]
	\centering
	{\tiny {
		\begin{tabular}{p{2.1cm}|p{1.6cm}|p{3.8cm}}
			\hline
			\textbf{Current Model Task}& \textbf{Output Label} & \textbf{Rule}\\
			\hline
			Object Detection & person & 2$\times$$\mathcal{P}$(Pose Estimation)\\
			Object Detection & person & 2$\times$$\mathcal{P}$(Gender Classification)\\
			Object Detection & dog & 2$\times$$\mathcal{P}$(Dog Classification)\\
			Face Detection & face & 2$\times$$\mathcal{P}$(Face Landmark Localization)\\
			Face Detection & face & 2$\times$$\mathcal{P}$(Emotion Classification)\\
			Pose Estimation & body keypoints & 2$\times$$\mathcal{P}$(Action Classification)\\
			Pose Estimation & wrist keypoints & 2$\times$$\mathcal{P}$(Hand Landmark Localization)\\
			Place Classification & indoor places & 0.5$\times$$\mathcal{P}$(Animal-Object Detection)\\
			Place Classification & indoor places & 0.5$\times$$\mathcal{P}$(Sport-Action Classification)\\
			\hline
		\end{tabular}
		\caption{Ten handcrafted model execution rules.}
		\label{tab:rules}
		}}
	\vspace{-0.12in}
\end{table}

Let $\mathcal{P}$(Task) denote the probability of executing models for certain ``Task''.
Initially, all models have equal execution probabilities.
After obtaining certain labels, the policy updates the execution probabilities of models according to the rules.
Fig.~\ref{fig:rule} shows the experimental results on MSCOCO 2017. 
DuelingDQN agent significantly outperforms the rule-based policy, which only saves 22.6\% (or 2.1\%) executed models and 20.1\% (or 1.4\%) running time with 0.8 (or 1.0) recall rate. Results on other datasets are similar.  

To understand the difference between the handcraft rule and the semantic relationship learned by the DRL agent,
 we visualize model execution sequences that are scheduled by the Q-value greedy policy of DuelingDQN agent on different sample images.
Fig.~\ref{fig:seq} shows an example of model execution sequence for an image (im6696.jpg) in MirFlickr25 dataset.
The knowledge of DRL agent meets the common sense that ``there should be some \textit{cups} in the \textit{pub} and people may \textit{drink beer} here''.
We can infer that the implicit semantic relationship among labeling capacities of deep learning models are highly complicated.
Even though the designers of the rules in Table~\ref{tab:rules} have some knowledge of deep learning models, the handcrafted rules only consider the pair-wise relationship and the influence on execution probabilities are fixed (2 or 0.5 times probability).
The more rules we designed, the harder to adjust the influence of each rule.
So we hold that designing a set of effective execution rules is infeasible for data labeling tasks at scale.


\begin{figure}[t]
    \centering
	\includegraphics[width=0.9\linewidth]{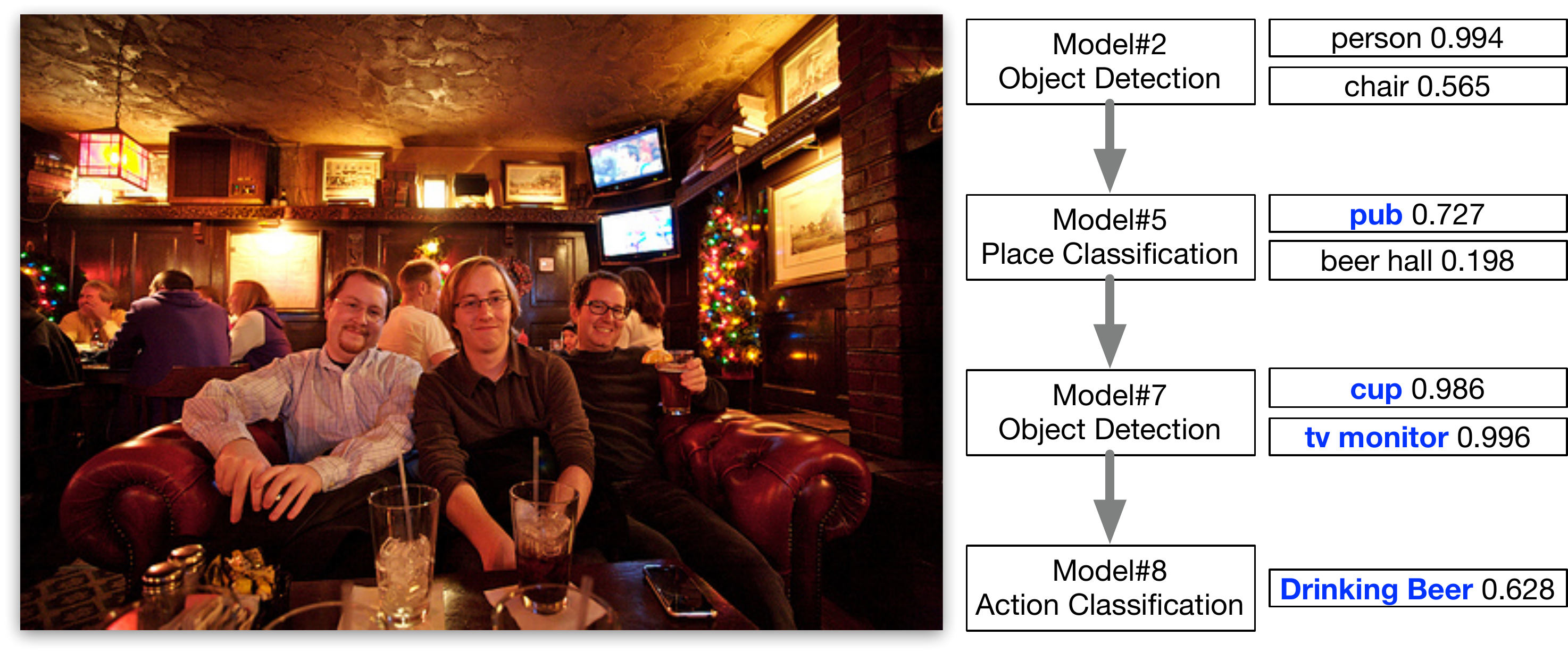}
	\caption{Model execution sequence of an image in MirFlickr25, scheduled by DuelingDQN agent with Q-value greedy policy.}
	\label{fig:seq}
	\vspace{-0.18in}
\end{figure}

\subsection{Knowledge Transferability}
\label{subsec:exp4}

\begin{figure*}[t]
	\begin{minipage}{.25\linewidth}
		\centering
		\includegraphics[width=\linewidth]{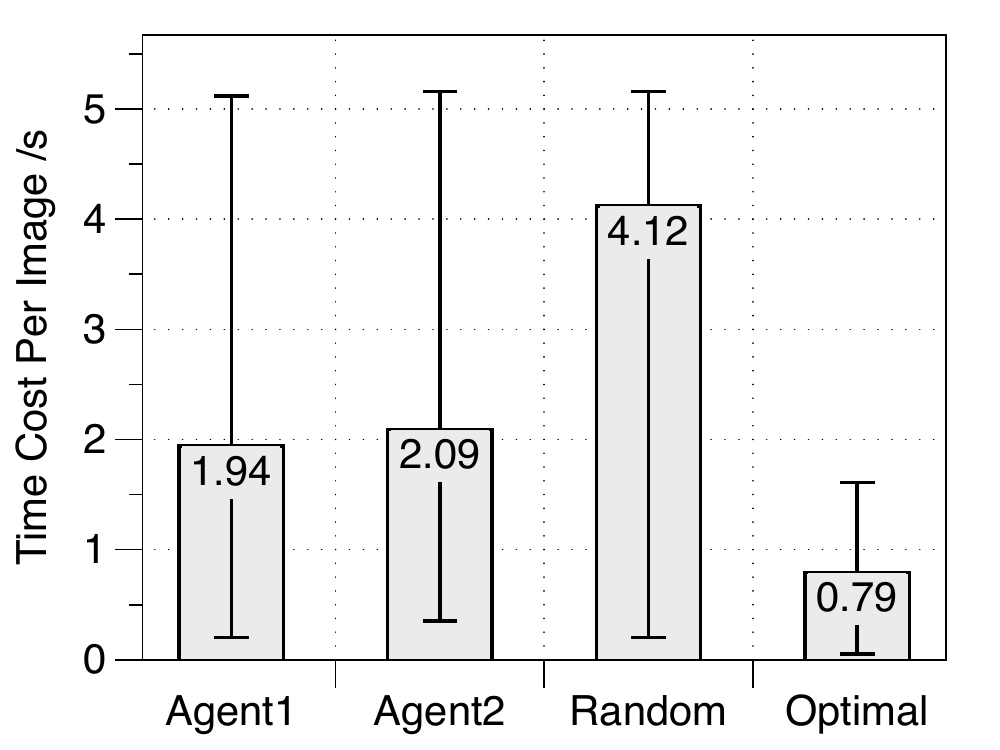}
		\textbf{\small (a) Average Time Cost \\ $Dataset_1$}
	\end{minipage}\hfill
	\begin{minipage}{.25\linewidth}
		\centering
		\includegraphics[width=\linewidth]{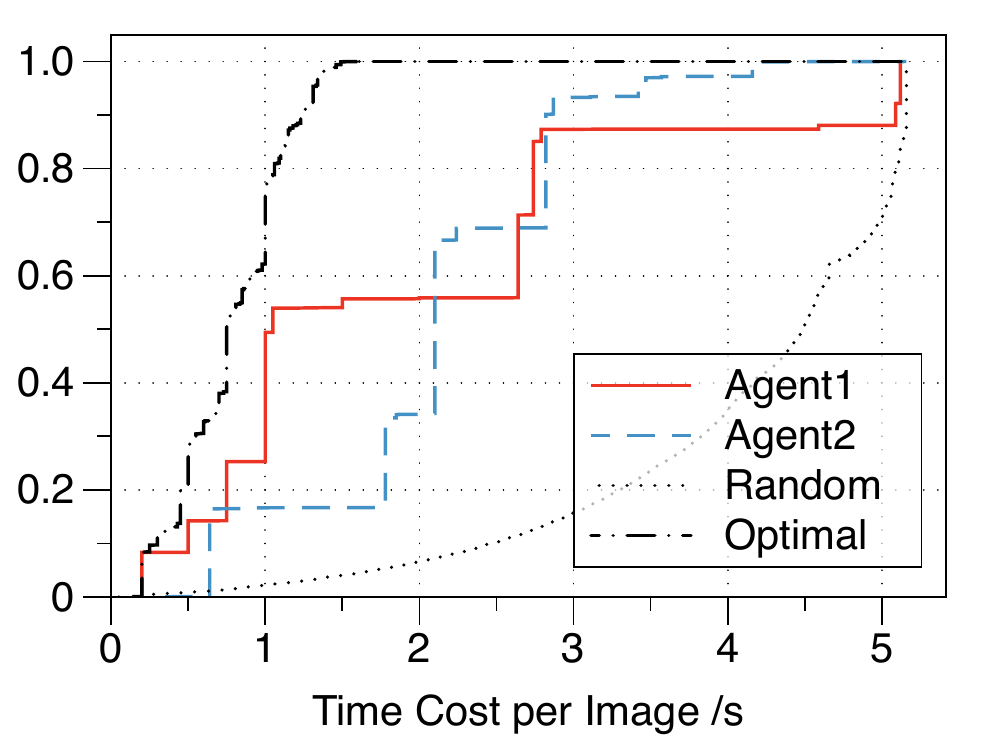}
		\textbf{\small (b) Time Cost CDF  \\ $Dataset_1$}
	\end{minipage}\hfill
	\begin{minipage}{.25\linewidth}
		\centering
		\includegraphics[width=\linewidth]{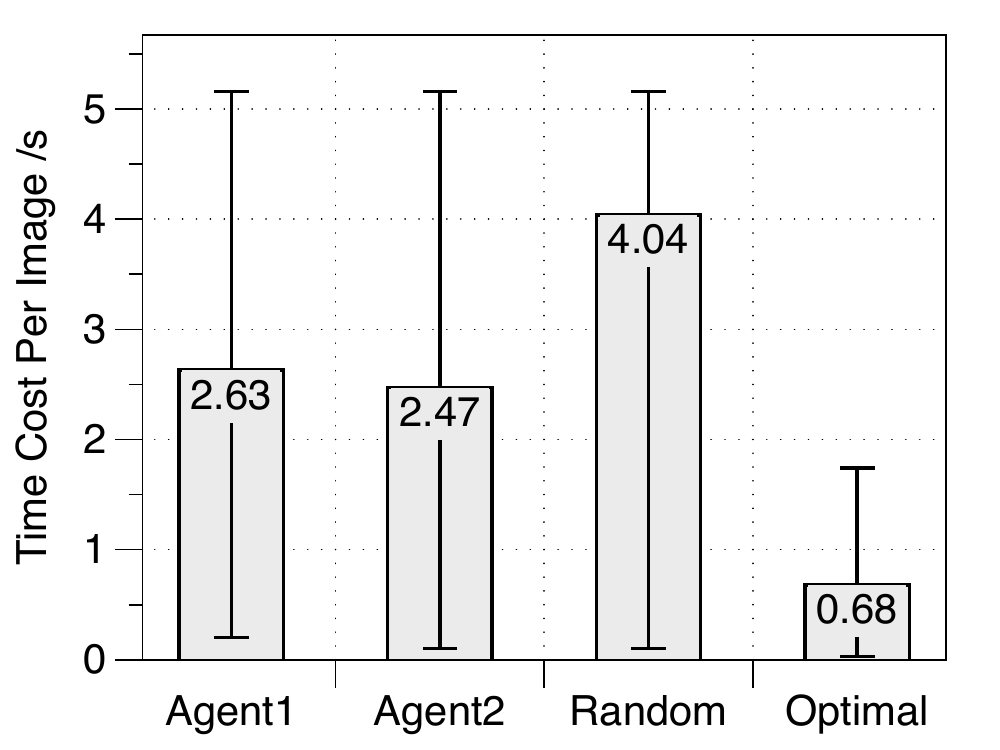}
		\textbf{\small (c) Average Time Cost \\$Dataset_2$}
	\end{minipage}\hfill
	\begin{minipage}{.25\linewidth}
		\centering
		\includegraphics[width=\linewidth]{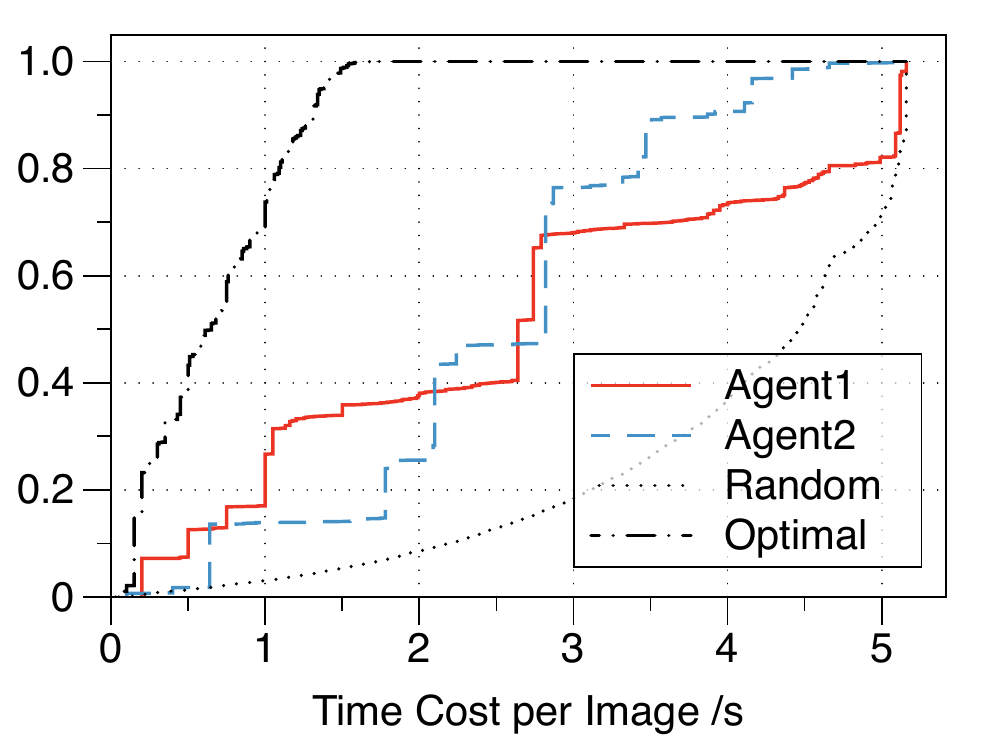}
		\textbf{\small (d) Time Cost CDF \\ $Dataset_2$}
	\end{minipage}
	\caption{Model execution time cost by different polices, tested on $Dataset_1$ and $Dataset_2$.}
	\vspace{-0.12in}
	\label{fig:transfer}
\end{figure*}

The experiments above show the effectiveness of DRL agents that were trained and tested on subsets of the same dataset.
An interesting question is: how \textit{transferable} the learned knowledge is?
Stanford40 and PASCAL VOC 2012 are two datasets that have obvious variance in content distribution.
Stanford40\cite{yao2011human} is collected for the action recognition task which mainly consists of scenes related to human activities.
PASCAL VOC 2012 dataset\cite{everingham2010pascal} covers a larger range of visual objects, including animals, vehicles, household furniture, etc.
Since DuelingDQN outperforms other DRL schema (\S\ref{subsec:exp2}), we use the experimental results of DuelingDQN agents as the representative. 
Two DuelingDQN-based agents are trained on the training samples from Stanford40 and PASCAL VOC 2012 respectively and tested them on both testing sets of the two datasets.
In this experiment, we use the average model execution time cost when all output value is recalled to measure the effectiveness of the agent.
The scheduling policy of DRL agents is still the Q-value greedy policy.
Random policy and optimal policy are also utilized as comparisons.
For simplicity, we use the following notations:

{\small
\begin{itemize}
	\item $Dataset_1$: Stanford40 testing set;
	
	\item $Datasset_2$: PASCAL VOC 2012 testing set;
	
	\item $Agent_1$: DRL agent trained on Stanford40 training set;
	
	\item $Agent_2$: DRL agent trained on PASCAL VOC 2012 training set.
\end{itemize}
\vspace{-0.05in}
}

As shown in Fig.~\ref{fig:transfer}, executing all 30 models costs 5.16s per image and the random policy averagely takes 4.12s on $Dataset_1$ and 4.04s on $Dataset_2$.
Compared with the random policy, 
 $Agent_1$ costs 1.94s (or 2.63s) on $Dataset_1$ (or $Dataset_2$).
 $Agent_2$ costs 2.09s (or 2.47s) on $Dataset_1$ (or $Dataset_2$).
So on average, DRL agents save 51.1\% running time on $Dataset_1$  and 36.9\% running time on $Dataset_2$.
It is inspiring that for one dataset the knowledge learned from the other dataset with different content distribution can also optimize scheduling effectively. 
Cumulative distribution figures (CDF) in Fig.~\ref{fig:transfer} show more details about the time cost distributions.
We can infer that although the knowledge learned from the two datasets vary a lot, 
 both of them could be utilized to optimize the model scheduling on a wide range of image datasets efficiently.

\textbf{Limitations:}
Although the experimental results demonstrate the transferability of the learned semantic relationship among models by DRL agents, 
we want to highlight some limitations. 1) Assumption of the intersected content distribution.
Stanford40 and PASCAL VOC have quite different content distributions, but effective optimization is still dependent on the assumption.
We studied some extreme cases, for example, training agents only by dog-related images and testing them by human action-related images, which show worse model scheduling than the random policy and vice versa. 2) The requirement of relatively adequate training samples.
Similar to other learning-based models, an effective DRL agent has some requirements for the volume of training samples.
An empirical size of the training set is around 10-20\% of the whole dataset with a random sampling method, though there is no theoretical proof of the convergence.


\subsection{Model Priority}
\label{subsec:exp5}
As introduced in Section~\ref{sec:semantic}, we use the parameter $\theta$ to control model priorities.
It is worth mentioning that the experiments above were conducted with identical $\theta$ value (1.0) for every model.
To study the effect of $\theta$, we train DRL agents by setting the parameter of one ``face detection model'' to 1.0, 2.0, 5.0 and 10.0 to increase its priority, 
 while keeping $\theta$ of other models as 1.0. 
Intuitively, with a higher $\theta$ parameter, we expect the DRL agent to predict a higher Q value for the ``face detection'' model in the early stages.
So we analyze the order of the ``face detection model'' in the scheduling sequences.
\begin{figure}[h]
	\vspace{0em}
	\centering
	\begin{minipage}{.5\linewidth}
		\centering
		\includegraphics[width=\linewidth]{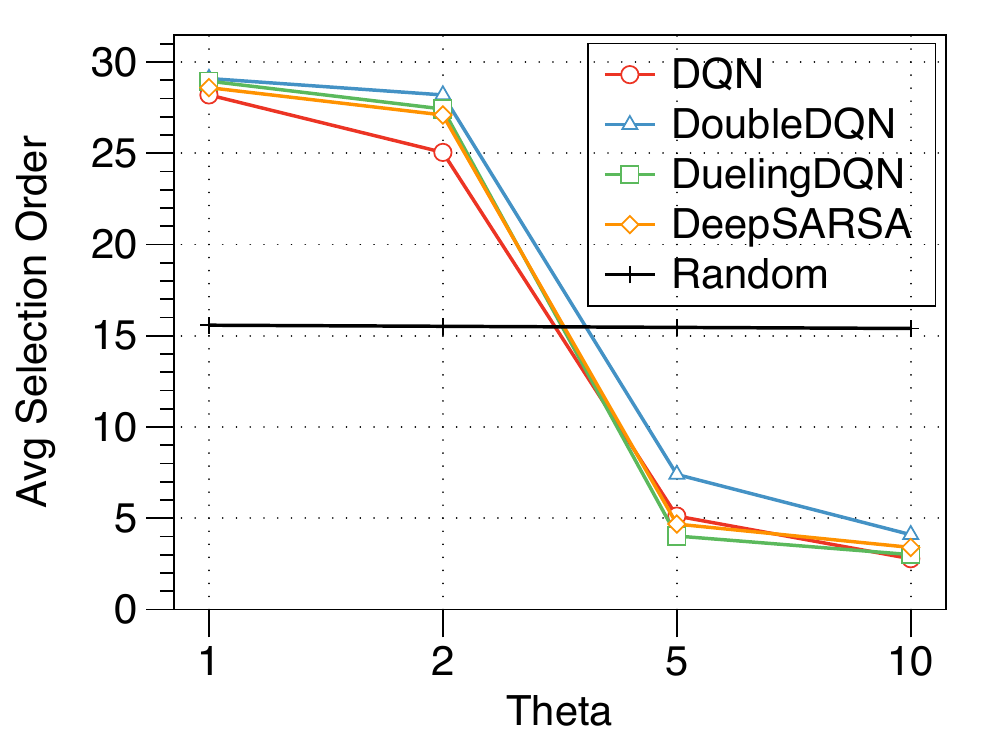}
		\textbf{\small (a) Average execution order}
	\end{minipage}\hfill
	\begin{minipage}{.5\linewidth}
		\centering
		\includegraphics[width=\linewidth]{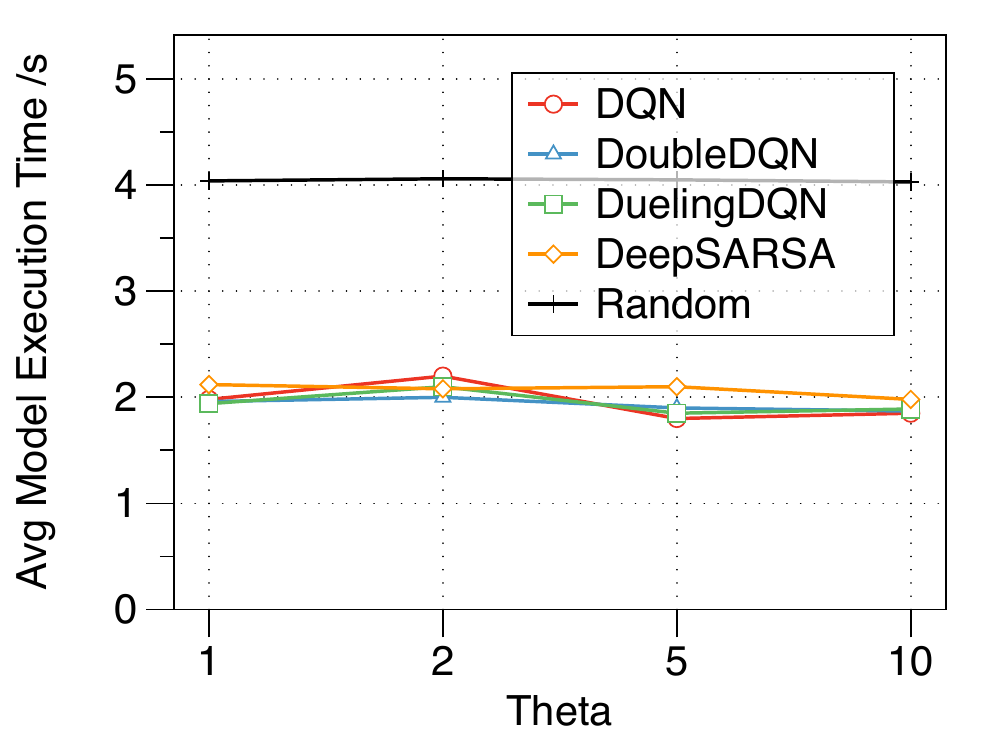}
		\textbf{\small (b) Average time cost}
	\end{minipage}
	\vspace{-0.12in}
	\caption{The effect of adjusting priority parameter $\theta$ of ``face detection'' model.}
	\label{fig:theta}
\end{figure}

Fig.~\ref{fig:theta}(a) shows that the increased $\theta$ effectively brings forward the execution of the ``face detection'' model.
It means that the user can obtain his/her preferred label (``face'' in this experiment) with a shorter delay.
As a representative, DuelingDQN agents schedule the ``face detection'' model at the 28.9 / 27.4 / 4.0 / 3.0 iterations on average, with 1.0 / 2.0 / 5.0 / 10.0 $\theta$ value.
Meanwhile, as shown in Fig.~\ref{fig:theta}(b), although $\theta$ shifts the order of model execution sequences, 
 the DRL agents still keep a good optimization performance on the total execution time (with 1.0 recall rate of label value).
On average, the DuelingDQN agents save 51.9 / 48.2 / 54.3 / 53.1\% running time without any loss of valuable labels, compared with the random policies.

\textbf{Practical utility.}
Being able to control the model execution priorities, while keeping the excellent scheduling performance, is quite valuable in data labeling applications.
For example, in a surveillance video monitoring system, we wish to obtain the output of the ``abnormal action detection'' model with higher priority (thus shorter delay) than the ``common object localization'' model, when the computing resources are limited.
However, simply using a fixed scheduling policy will cause a serious waste of resources, since abnormal actions occur much less frequently than common objects.
Then the ability of the proposed DRL agent is crucial to solving the challenge:
flexibly controlling the model priorities and keeping effective optimization of scheduling.


\subsection{Scheduling under Deadline Constraint}
\label{subsec:exp6}

\begin{figure*}[t]
    \centering
    \begin{minipage}{.25\linewidth}
        \centering
		\includegraphics[width=\linewidth]{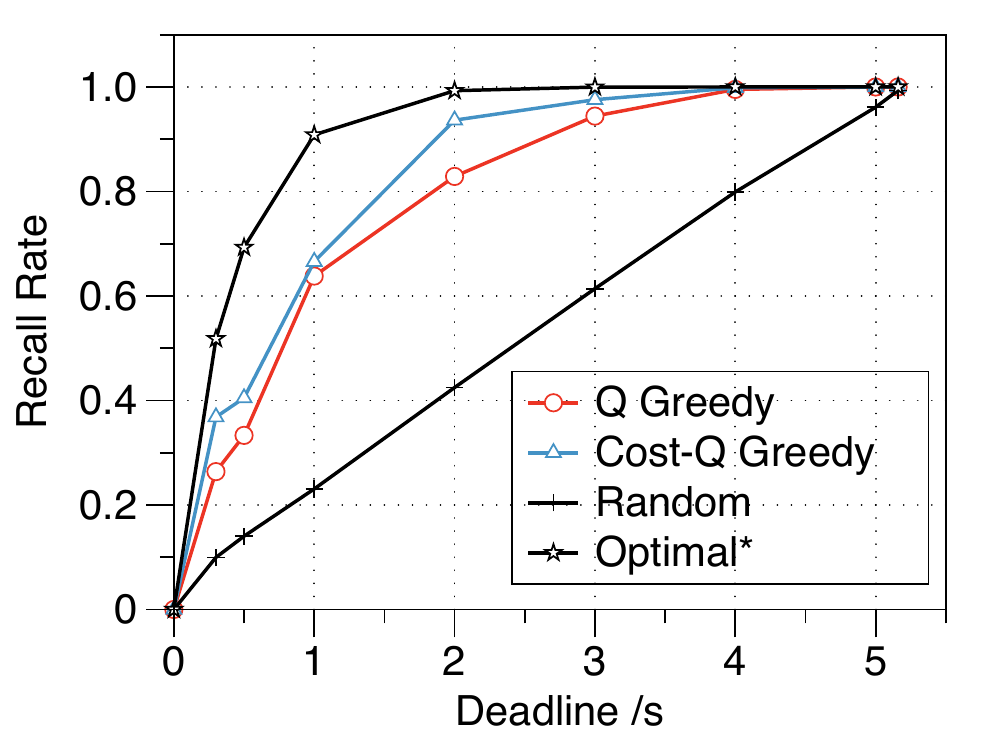}
		\textbf{\small (a) MSCOCO 2017}
	\end{minipage}\hfill
	\begin{minipage}{.25\linewidth}
		\centering
		\includegraphics[width=\linewidth]{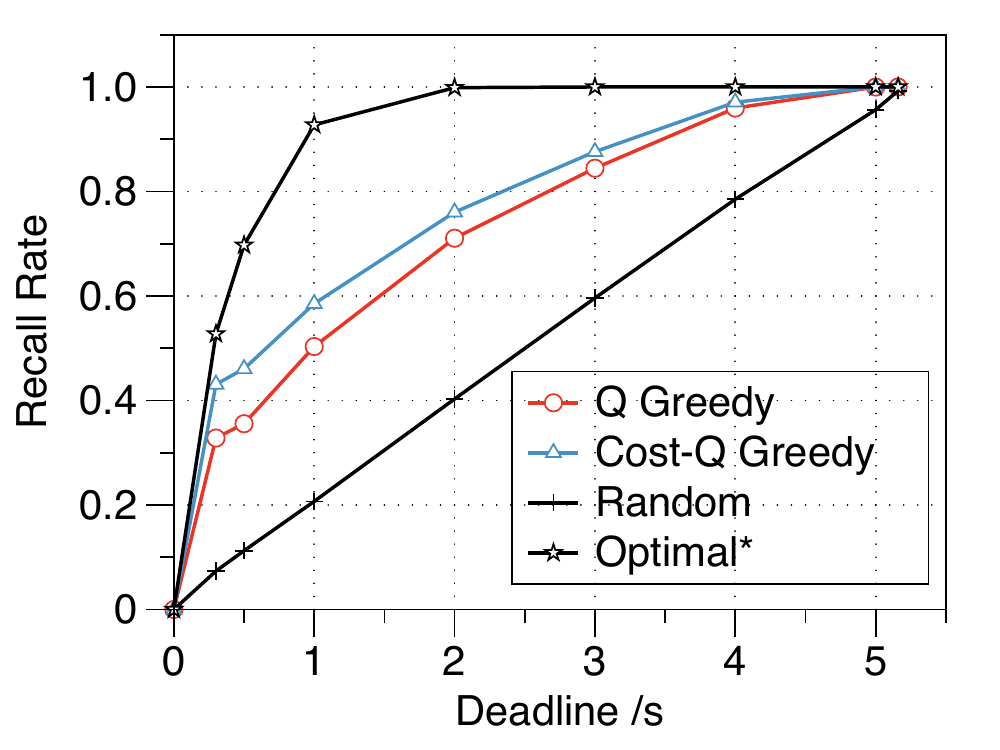}
		\textbf{\small (b) MirFlickr25}
	\end{minipage}\hfill
	\begin{minipage}{.25\linewidth}
		\centering
		\includegraphics[width=\linewidth]{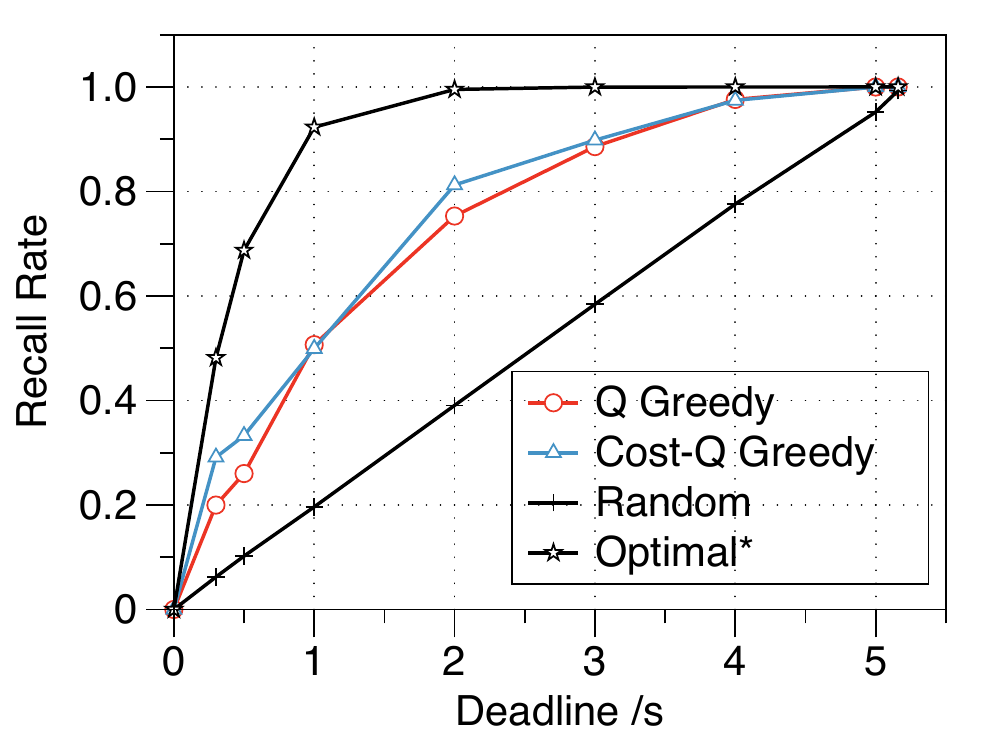}
		\textbf{\small (c) Places365}
	\end{minipage}\hfill
	\begin{minipage}{.25\linewidth}
	    \centering
    	\includegraphics[width=\linewidth]{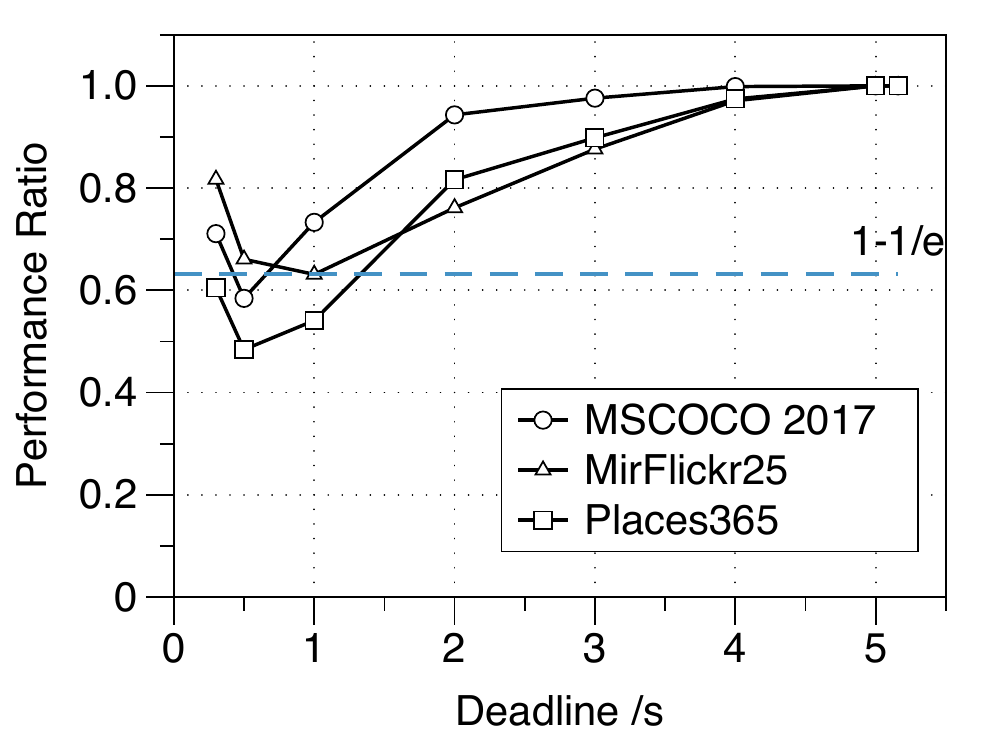}
    	\textbf{\small (d) Performance ratio}
	\end{minipage}
	\caption{Value recall rate under deadline constraints.}
	\label{fig:ddl}
	\vspace{-0.12in}
\end{figure*}

\begin{figure*}[t]
	\vspace{0em}
	\begin{minipage}{.25\linewidth}
		\centering
		\includegraphics[width=\linewidth]{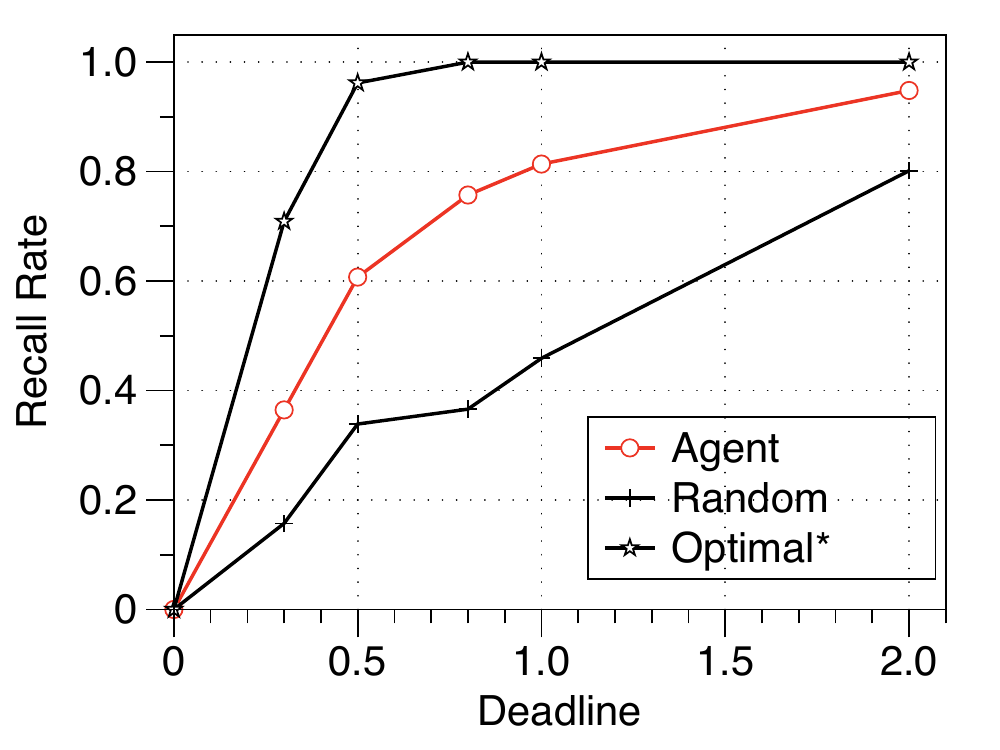}
		\textbf{\small (a) 8GB Memory }
	\end{minipage}\hfill
	\begin{minipage}{.25\linewidth}
		\centering
		\includegraphics[width=\linewidth]{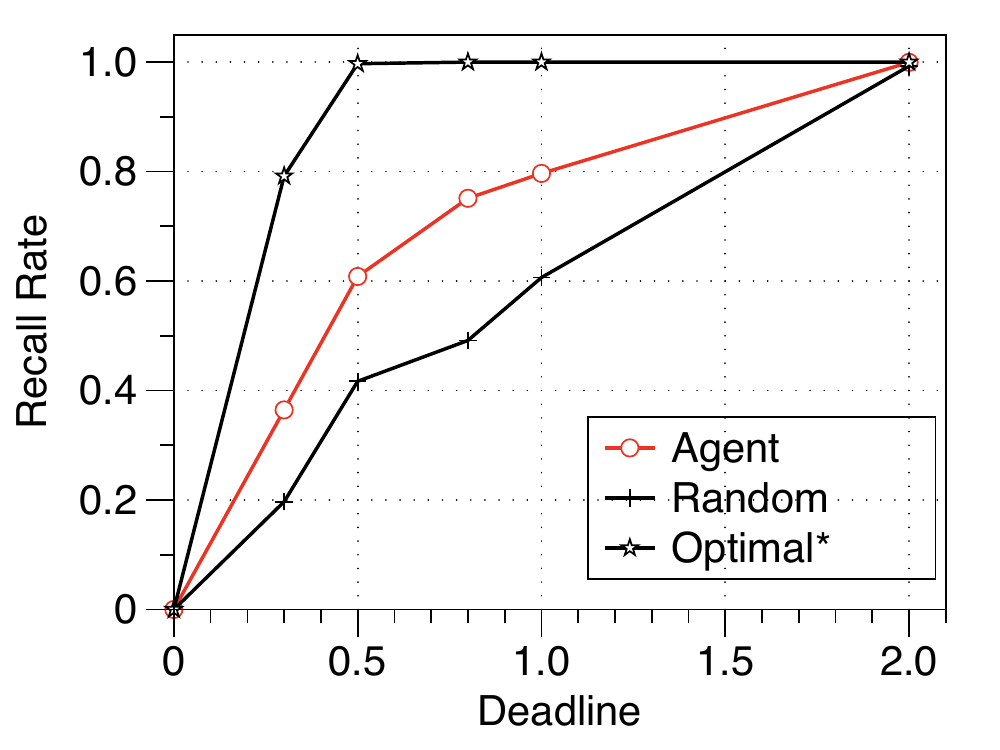}
		\textbf{\small (b) 12GB Memory}
	\end{minipage}\hfill
	\begin{minipage}{.25\linewidth}
		\centering
		\includegraphics[width=\linewidth]{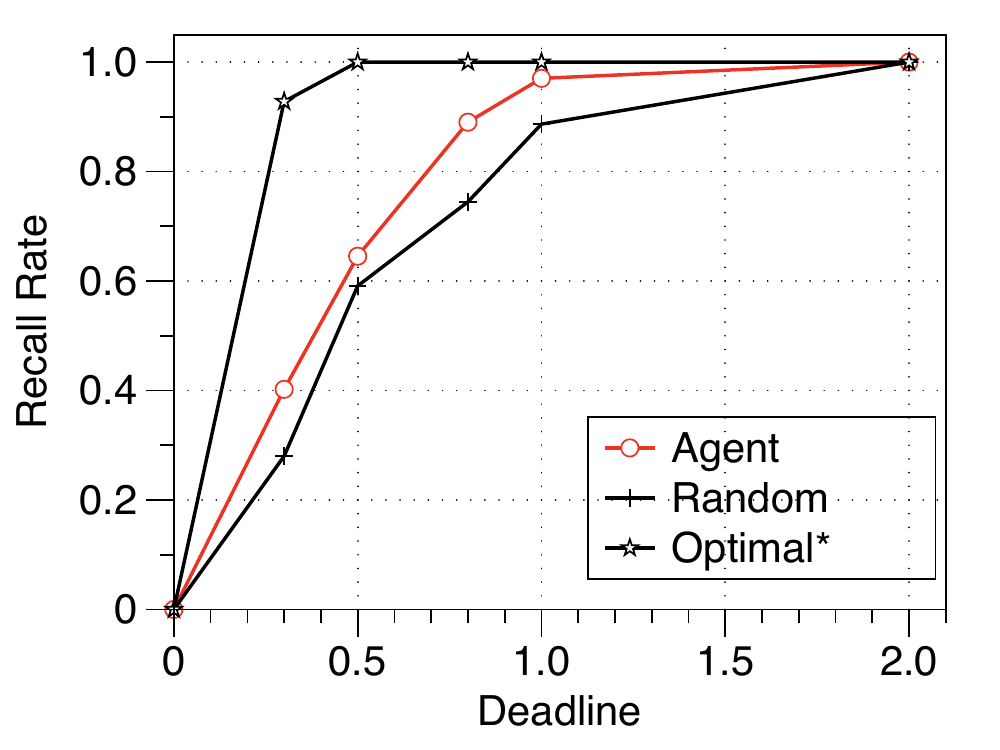}
		\textbf{\small (c) 16GB Memory}
	\end{minipage}\hfill
	\begin{minipage}{.25\linewidth}
		\centering
		\includegraphics[width=\linewidth]{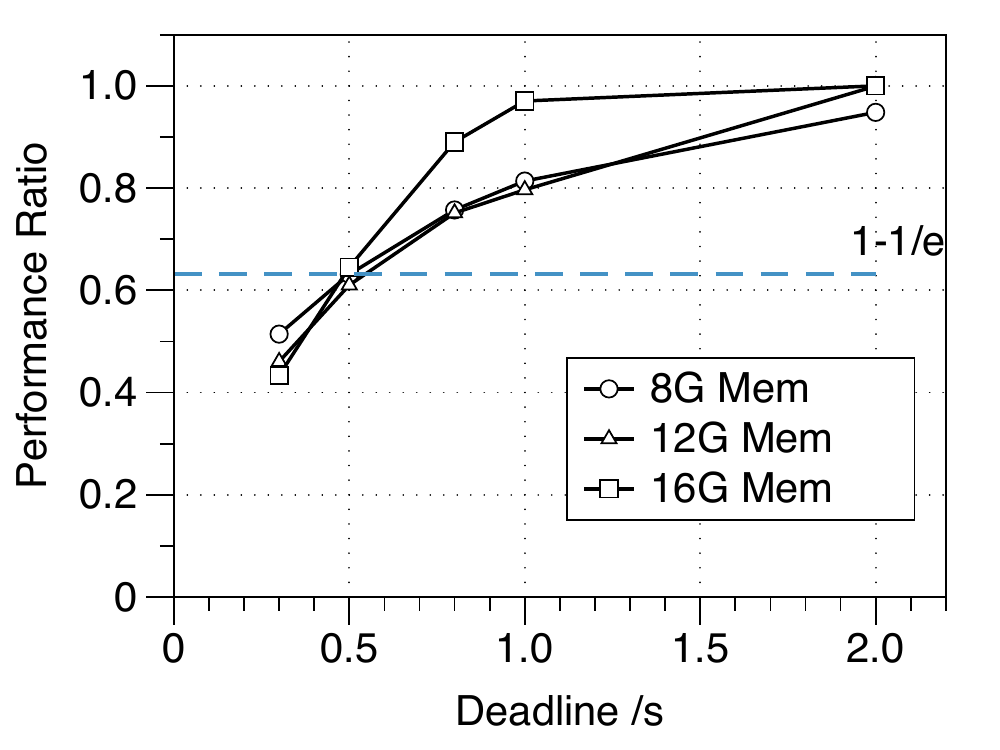}
		\textbf{\small (d) Performance ratio.}
	\end{minipage}
	\caption{Value recall rate under different memory and deadline constraints.}
	\vspace{-0.12in}
	\label{fig:ddl-mem}
\end{figure*}

Experiments above use the required output value recall rate as the condition to terminate the scheduling process.
More practical usage of the proposed framework is scheduling models under computing resources or delay constraints.
We consider the most common constraint, the deadline of each input data, of data labeling tasks, and evaluate the proposed scheduling Algorithm~\ref{alg:ddl} (referred to as \textit{Cost-Q Greedy policy}).
Using the output value recall rate under deadline constraint as the metric, three policies are implemented for comparison:

{\small
\begin{itemize}
\item     \textbf{Random policy}: randomly selects model until the deadline;
\item     \textbf{Optimal* policy}: greedily selects the model with maximal \\
$\frac{f(S\cup\{m\},d) - f(S,d)}{m.time}$ with relaxing constraint (\S\ref{subsec:algo});
\item     \textbf{Q-Greedy policy}: greedily selects the model with maximal Q value until the deadline (\S\ref{subsec:exp2}). 
\end{itemize}}

As a representative, all DRL agents in the following experiments are trained with DuelingDQN schema.
Fig.~\ref{fig:ddl}(a-c) show that Algorithm~\ref{alg:ddl} outperforms the Q-greedy policy on three datasets.
Algorithm~\ref{alg:ddl} boosts the value recall rate by 188.7-309.5\% with a 0.5s delay budget, compared with the random policies.
We plot the performance ratio of Algorithm~\ref{alg:ddl} to optimal* policy in Fig.~\ref{fig:ddl}(d), which shows that in most cases our cost-Q greedy algorithm performs better than the provable guarantee $1-1/e$ of classic approaches.

Moreover, we evaluate $Agent_1$ and $Agent_2$ on $Dataset_1$ and $Dataset_2$ (\S\ref{subsec:exp4}), using Algorithm~\ref{alg:ddl} instead of Q-greedy policy.
As shown in Fig.~\ref{fig:ddl2}, the DRL agent knowledge is transferable between dataset with different content distributions.
With 1.0s deadline constraint, $Agent_1$ and $Agent_2$ improve the recalled value by 346.8 (or 250.5\%) and 224.9 (or 190.5\%) on $Dataset_1$  (or $Dataset_2$).

\begin{figure}[t]
	\begin{minipage}{.5\linewidth}
		\centering
		\includegraphics[width=\linewidth]{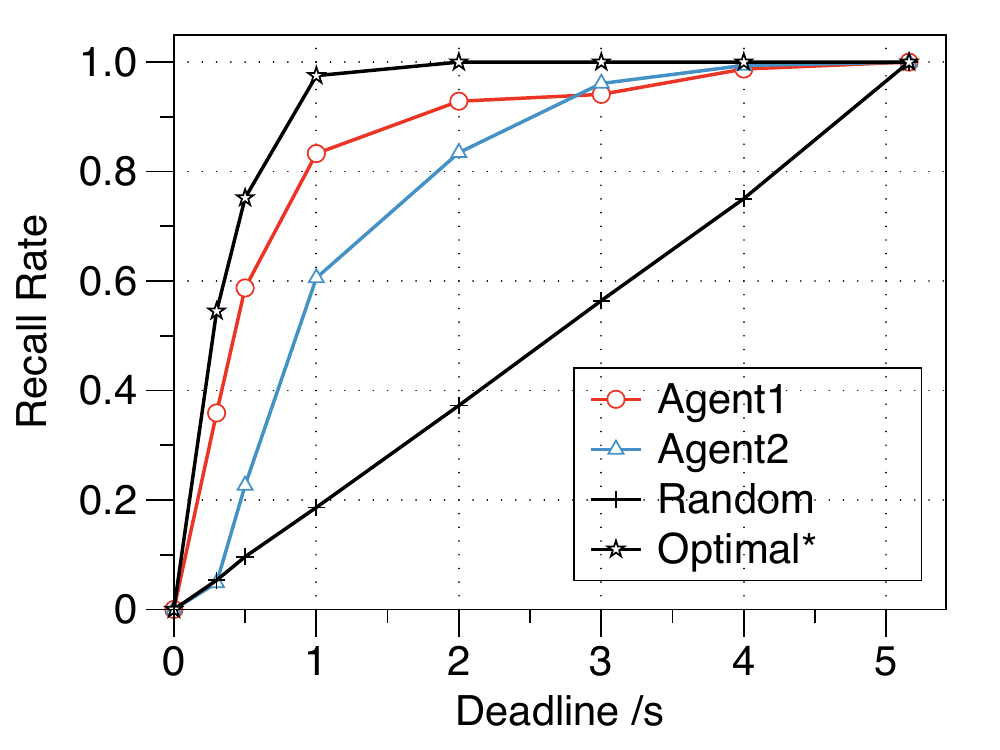}
		\textbf{\small (a) $Dataset1$}
	\end{minipage}\hfill
	\begin{minipage}{.5\linewidth}
		\centering
		\includegraphics[width=\linewidth]{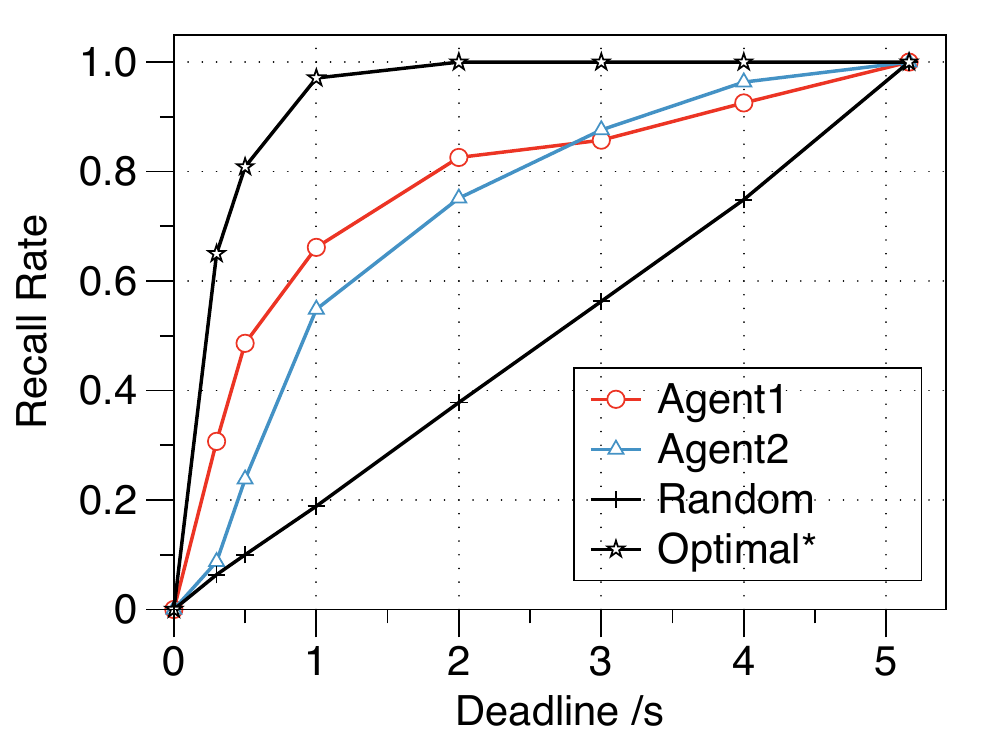}
		\textbf{\small (b) $Dataset2$}
	\end{minipage}
	\vspace{-0.1in}
	\caption{Value recall rate under deadline constraints.}
	\label{fig:ddl2}
\end{figure}

\subsection{Scheduling under~Memory-Deadline~Constraints}
\label{subsec:exp7}

To tackle the more challenging problem, scheduling models under two-dimension knapsack constraints, we propose Algorithm~\ref{alg:ddl-mem}.
Deadline and GPU memory constraints are commonly limited computing resources which are \textit{orthogonal}:
GPU memory restricts the \textit{spatial} size of the parallel deep learning models, while deadline limits the overall running time in \textit{temporal} dimension.
Identically, we use the output value recall rate as the metric.
Random policy and optimal* policy are used as baselines:
{\small
\begin{itemize}
    \item \textbf{Random policy}: randomly selects model that could be packed into GPU to execute until the deadline;
    \item \textbf{Optimal* policy}: greedily selects the model with maximal \\
    $\frac{f(S\cup\{m\},d)-f(S,d)}{m.time * m.mem}$ with relaxing constraints (\S\ref{subsec:algo});
\end{itemize}
}

As a representative, we use the DuelingDQN $Agent_1$ and $Dataset_2$ as the test set, which are the worst cases in our experiments.
As shown in Fig.~\ref{fig:ddl-mem}(a-c), Algorithm~\ref{alg:ddl-mem} significantly improves the output value recall rate compared with random policies.
More specifically, the recall rate of output value is improved by 106.9\% / 52.8\% / 19.5\% under 8GB / 12GB / 16GB GPU memory and 0.8s deadline constraints.
With the increasing memory allocated for learning models, the room for improvement between random policy and optimal* policy shrinks. 
So it is reasonable that the improvement brought by Algorithm~\ref{alg:ddl-mem} is relatively small with 12GB and 16GB memory.
As shown in Fig.~\ref{fig:ddl-mem}(d), the performance ratio of Algorithm~\ref{alg:ddl-mem} to optimal* policy exceeds $1-1/e$ in most cases.

\subsection{Scheduling Overhead}
\label{subsec:exp8}
Experimental results above illustrate the effectiveness of DRL agents and adaptive model scheduling algorithms.
Here, we measure the additional overhead brought by our framework.
Fortunately, as shown in Table~\ref{tab:overhead}, the time cost for making one selection is only around 3-6ms,  which is negligible compared with the execution time of deep learning models (50-400ms).
For memory usage, the DRL agents require about 100MB CPU memory, which is quite lightweight compared with the deployed visual analysis models that consume 500MB to 8GB GPU memory.
\begin{table}[h]
    \vspace{-0.1in}
	\begin{center}
	{\small
		\begin{tabular}{p{1.5cm}|p{2cm}|p{3.5cm}}
			\hline
			&\textbf{DRL Agent}&\textbf{Deep Learning Model}\\
			\hline
			Time& 3-6ms & 50-400ms\\
            \hline
			Memory& 100MB (CPU) & 500-8000MB (GPU)\\
			\hline
		\end{tabular}
		\caption{Computing cost of DRL agent and models.}
		\label{tab:overhead}
	}
	\end{center}
	\vspace{-0.2in}
\end{table}



%% file: version2/related.tex
Towards comprehensive and efficient data labeling, previous work can be divided into two main streams: enhancing one single model's ability and accelerating the model execution.

\subsection{Model Ability Enhancement}
\textbf{Multi-label learning.}
Many industrial applications, including document categorization and automatic image annotation, require to label the raw data with multiple tags.
Multi-label learning is proposed to enable models to output a set of labels for one input data.
The extreme output dimension and complex structures are the main challenges of multi-label learning tasks~\cite{xu2019survey}.
Han et al.~\cite{han2015learning} proposed a collaborative embedding approach, which exploits the association between embedding features and annotations.
Yang et al.~\cite{yang2018complex} designed a NNs which can learn to predict labels and exploit correlation among them simultaneously.  

\textbf{Multi-task learning.}
Another line of work studies multi-task learning (MTL)~\cite{ruder2017overview} that aims to leverage implicit relatedness among different tasks to improve the performance of all of them.
Hard and soft parameter sharing are the most commonly used technologies in MTL.
Kaiser et al.~\cite{kaiser2017one} proposed a deep learning model that can learn multiple tasks on different data types, including image captioning (image data), speech recognition (audio data) and English parsing (text data).
Recent work~\cite{ma2018modeling} designed a method to automatically learn to model the task relationships directly from data, without human expertise.

Multi-label learning and multi-task learning exploit the correlation among output labels and related tasks and successfully enhance a single model's ability.
However, the increasing complexity of industrial tasks, where the more extracted information the better, results in the insufficiency of a single model.
As an effective complement, our proposed adaptive model scheduling framework leverages the power of multiple models with minimum computation waste.

\subsection{Model Execution Acceleration}
\textbf{Model compression.}
Existing literature provides many technologies for deep learning model compression.
Parameter pruning and sharing~\cite{chen2016compressing}
is one popular technology to reduce redundant parameters and operations in deep neural networks.
Another line of work~\cite{wu2017squeezedet}
focuses on designing efficient convolutional filters, which is crucial to the execution speed of visual analysis models.

\textbf{Adaptive model configuration.}
Except for improving the computing efficiency, one recent work~\cite{jiang2018chameleon} found optimization potential from the adaptive model configuration.
They consider configurable variables like the video frame rate, resolution, and neural networks architecture and proposed a framework to adaptively select configurations to accelerate a video analyzing system.

Different from directly accelerating model execution, our proposed adaptive model scheduling avoids executing unnecessary models by learning implicit relationships among different models.
We believe our work is a powerful complement to the existing optimization technologies.

%% file: version2/conclusion.tex
In this work, we tackled a challenging task, adaptive model scheduling, which works as an effective approach towards comprehensive and efficient data labeling.
We designed a framework, including a novel method to predict unexecuted models' value 
 and adaptive scheduling algorithms to improve the aggregated values of executed models for each data item.
Our extensive evaluations demonstrate that our design achieves significant performance improvement compared with other approaches.  
Several challenging issues are left as future work.
A critical innovative component of our framework is the propose and construction of the model-relationship graph.
Firstly, we would like to design a fast method to construct this efficiently and effectively. 
Secondly, designing scheduling algorithms with theoretical performance guarantees under different constraints remains a very challenging and attractive problem.
Recall that many notoriously difficult to tackle NP-hard problems are special cases of our scheduling problem.
At last, we also need to further evaluate our proposed method on other data types like text and audio.